\documentclass{article}

\PassOptionsToPackage{numbers,sort}{natbib}
\usepackage[preprint]{neurips_2026}

\usepackage{amssymb}
\usepackage[utf8]{inputenc} % allow utf-8 input
\usepackage[T1]{fontenc}    % use 8-bit T1 fonts
\usepackage{hyperref}       % hyperlinks
\usepackage{url}            % simple URL typesetting
\usepackage{booktabs}       % professional-quality tables
\usepackage{amsfonts}       % blackboard math symbols
\usepackage{nicefrac}       % compact symbols for 1/2, etc.
\usepackage{microtype}      % microtypography
\usepackage{xcolor}         % colors
\usepackage{graphicx}
\usepackage{amsmath}
\usepackage{algorithm}
\usepackage{algpseudocode}
\usepackage{amsthm} 
\usepackage{wrapfig}
\usepackage{enumitem}

\usepackage{booktabs} 
\usepackage{multirow}   
\usepackage{xcolor}     
\usepackage{colortbl}
\usepackage{subcaption}
\usepackage{makecell}

\definecolor{highlightblue}{RGB}{229, 241, 255} 
\definecolor{lgray}{RGB}{192, 192, 192}

\newtheorem{asmp}{Assumption}

\newtheorem{theorem}{Theorem}[section] 
\newtheorem{lemma}[theorem]{Lemma}
\newtheorem{proposition}[theorem]{Proposition}

\algnotext{EndFor}

\title{LENS: Low-Frequency Eigen Noise Shaping for Efficient Diffusion Sampling}
\author{
\begin{tabular}{cc}
Haewon Jeon & Si-Hyeon Lee \\
\end{tabular}
\\[0.5em]
School of Electrical Engineering \\
Korea Advanced Institute of Science and Technology (KAIST) \\
Daejeon, South Korea \\
\texttt{\{haewon, sihyeon\}@kaist.ac.kr}
}

\begin{document}

\maketitle

\begin{abstract}
    Distilled diffusion models accelerate image generation by reducing the number of denoising steps, but often suffer from degraded image quality. To mitigate this trade-off, test-time optimization methods improve quality, yet their iterative nature incurs substantial computational overhead and leads to slow inference, limiting practical usability. Recent hypernetwork-based approaches amortize this process during training, but still require costly noise modulation in high-dimensional latent spaces. In this work, we propose \textbf{LENS} (Low-frequency Eigen Noise Shaping), an efficient noise modulation framework that operates in a low-dimensional subspace. Our approach is motivated by the observation that low-frequency components of the noise largely determine the global structure and visual fidelity of generated images. Based on this observation, we provide a theoretical justification for restricting modulation to the low-frequency subspace and derive a principled training objective. Building on this, LENS employs a lightweight, standalone network to selectively modulate these components, enabling efficient and targeted noise modulation. Extensive experiments demonstrate that LENS achieves competitive image quality while reducing FLOPs by 400–700$\times$, model parameters by 25–75$\times$, and inference-time overhead by 10–20$\times$ compared to prior methods. The project page is available at \url{https://jeonhaewon.github.io/LENS-project/}.
    %Distilled diffusion models alleviate the slow generation speed of standard diffusion models by reducing the number of iterative denoising steps, but this efficiency comes at the cost of degraded image quality. To address this trade-off, test-time optimization methods have been introduced to improve generation quality; however, their high computational cost and slow inference significantly limit practical usability. Recent studies have explored hypernetwork-based approaches as an alternative to test-time optimization, though their computational efficiency during noise modulation can be further improved. In this work, we propose LENS (Low-frequency Eigen Noise Shaping), which efficiently modulates noise by selectively adjusting low-frequency coefficients. Our approach is motivated by the observation that low-frequency components of noise play a crucial role in shaping the global structure and visual fidelity of generated images. LENS uses a lightweight, separate network to modulate initial noise at generation time, enabling higher-quality sampling with negligible overhead. Furthermore, we provide a theoretical analysis that justifies the effectiveness of adjusting low-frequency coefficients. Extensive experiments show that LENS achieves competitive image quality with significantly lower computational cost and faster inference than existing methods.
\end{abstract}

\section{Introduction}

Diffusion models have achieved remarkable success in high-quality text-to-image generation~\cite{zhang2023text, rombach2022high}. These models generate images by iteratively denoising an initial noise sample, but achieving high fidelity typically requires a large number of sampling steps, resulting in high computational cost and slow inference~\cite{sehwag2025stretching, esser2024scaling}. To improve efficiency, distilled diffusion models have been developed to enable generation in a single or a few steps~\cite{song2023consistency, yin2024one, xie2024distillation}. However, this efficiency often comes at the expense of generation quality, leading to an inherent trade-off between speed and fidelity~\cite{xie2024tlcm}. This motivates methods that enhance generation quality in distilled models without sacrificing their efficiency.

Recent work has explored inference-time techniques that improve generation quality by performing additional computations during inference while keeping model parameters fixed~\cite{uehara2025rewardguided, uehara2025inference}. Among these, a prominent direction is to optimize the initial noise, motivated by the observation that it governs the global structure of the generated image~\cite{xu2025good, mao2024lottery, lyu2026diff, ban2025crystal, li2025enhancing}. Conventional approaches perform test-time optimization, iteratively refining the noise for each sample~\cite{guo2024initno, eyring2024reno, karunratanakul2024optimizing, ben-hamu2024dflow, tang2024tuningfree, miao2025noise, wallace2023end, novack2024ditto}, which incurs substantial computational overhead.

To mitigate this, recent methods amortize the optimization by learning a prompt-conditioned mapping from noise to improved noise~\cite{li2025noisear, ahn2026noise, zhou2025golden, eyring2025noise}. HyperNoise~\cite{eyring2025noise} is a representative approach that predicts noise perturbations by reusing the diffusion backbone augmented with lightweight adaptation modules. Specifically, the initial noise is first passed through the Low-Rank Adaptation (LoRA)~\cite{hu2022lora}-augmented diffusion model to generate a modified noise representation, which is then fed into the original diffusion model for image synthesis. As a result, noise modulation requires executing the diffusion backbone itself, both during training (for learning the perturbation mapping) and at inference time (for each sample), leading to non-negligible computational overhead despite improvements over iterative optimization methods. Moreover, this process operates in the full high-dimensional latent space, introducing unnecessary degrees of freedom and further limiting efficiency.

Meanwhile, prior studies suggest that low-frequency components of the noise primarily determine global structure, whereas high-frequency components mainly affect fine details~\cite{wang2025seeds}. This insight indicates that effective noise modulation can be achieved in a low-dimensional subspace.

Motivated by these observations, we propose \textbf{Low-frequency Eigen Noise Shaping (LENS)}, a framework that performs efficient noise modulation by operating on low-frequency coefficients of the latent noise. By restricting modulation to this low-dimensional subspace, LENS enables the use of a lightweight network that is decoupled from the diffusion backbone, significantly reducing computational cost while preserving control over key structural attributes of generated images. As illustrated in Figure~\ref{intro_figure}, LENS achieves a favorable performance--efficiency trade-off, maintaining competitive generation quality with negligible additional inference cost.
%By decoupling noise modulation from the diffusion backbone and restricting optimization to a low-dimensional subspace, LENS significantly reduces computational cost while preserving control over key structural attributes of generated images.
\begin{figure} 
    \centerline{\includegraphics[width=0.92\columnwidth]{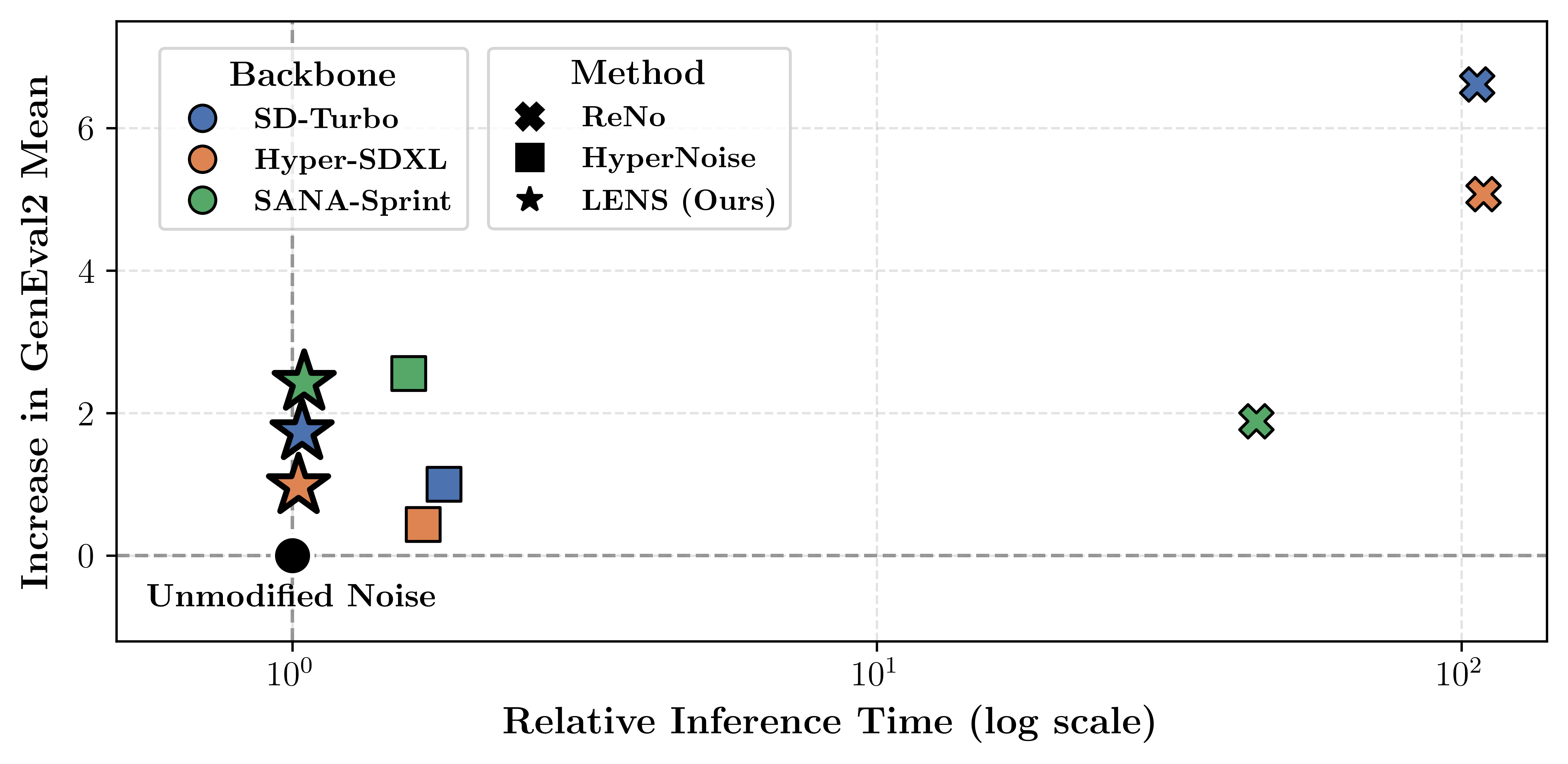}}
\caption{\textbf{Performance vs. inference time relative to unmodified noise input.} The x-axis shows inference time normalized by the case where the input noise is sampled directly from a standard Gaussian distribution without any modification, plotted in log scale. The y-axis shows the increase in GenEval2~\cite{kamath2025geneval} mean over this reference. Hence, the point $(1, 0)$ corresponds to this unmodified noise setting. 
We compare three approaches: ReNo~\cite{eyring2024reno}, which performs iterative test-time optimization of the input noise; HyperNoise~\cite{eyring2025noise}, which learns a noise modulation network using the diffusion backbone; and the proposed LENS. Results are shown across three distilled diffusion backbones: SD-Turbo~\cite{sauer2024adversarial}, Hyper-SDXL~\cite{ren2024hyper}, and SANA-Sprint~\cite{chen2025sana}. LENS achieves comparable or improved quality relative to HyperNoise with reduced inference time, while ReNo attains higher quality at a substantially higher computational cost.}
  %  \caption{\hw{\textbf{Relative performance vs. inference time normalized by the random baseline.} The x-axis shows the relative inference time $(value / random)$ in log scale, while the y-axis represents the increase in GenEval2~\cite{kamath2025geneval} mean $(value - random)$ compared to the random baseline (i.e., noise sampled from a standard Gaussian distribution without modulation). The point $(1, 0)$ represents the random baseline. We compare three methods—ReNo~\cite{eyring2024reno}, HyperNoise~\cite{eyring2025noise}, and LENS—across three backbone models: SD-Turbo~\cite{sauer2024adversarial}, Hyper-SDXL~\cite{ren2024hyper}, and SANA-Sprint~\cite{chen2025sana}. The LENS method incurs only a modest increase in generation time over Random Noise, while achieving performance comparable to or exceeding HyperNoise across all models. Furthermore, it remains substantially faster than ReNo while maintaining a favorable performance–time trade-off.}}
    \label{intro_figure}
%arxiv
%\vspace{-5pt}
\end{figure}

Our contributions are summarized as follows:
\begin{itemize}[leftmargin=20pt]
\item We propose LENS, a novel framework for efficient noise modulation in distilled diffusion models via optimization in a low-frequency subspace. By eliminating the need to reuse the diffusion backbone and reducing the effective degrees of freedom during training, LENS enables more efficient and stable learning. 
\item We provide theoretical justification for focusing on low-frequency components and derive a principled training objective in the corresponding coefficient space.

\item We demonstrate that LENS achieves comparable or improved generation quality while significantly reducing computational cost, with 400–700$\times$ fewer FLOPs, 25–75$\times$ fewer parameters, and 10–20$\times$ faster noise modulation.
\end{itemize}

%\newpage
\section{Related Work}
\label{related work}

\paragraph{Distilled Diffusion Models.}
Diffusion-based generative models produce high-quality images by iteratively denoising an initial noise sample $z_0 \sim \mathcal{N}(0, I)$ \cite{ho2020denoising, song2021denoising}. However, this process typically requires many sampling steps, resulting in high computational cost \cite{salimans2022progressive}. To address this limitation, recent advances in diffusion model distillation, including SD-Turbo \cite{sauer2024adversarial}, Hyper-SDXL \cite{ren2024hyper}, and SANA-Sprint \cite{chen2025sana}, have drastically reduced the number of sampling steps. Despite these improvements, such efficiency gains often come at the expense of generation quality, with distilled models typically underperforming their full diffusion counterparts. This trade-off motivates inference-time techniques that improve generation quality while retaining the efficiency of distilled models.

\paragraph{Noise Hypernetworks.} 
Recently, a novel approach has been proposed to improve generation quality by modulating the initial noise via a noise hypernetwork, enabling faster inference than existing iterative test-time optimization methods \cite{guo2024initno, eyring2024reno, karunratanakul2024optimizing, ben-hamu2024dflow}. Given a pretrained generator $G_\theta$ that maps an initial noise sample $z_0$ to an output image $x$, i.e., $x = G_\theta(z_0)$, let $r(x)$ be a reward function that quantifies the quality of generated images. HyperNoise \cite{eyring2025noise} considers a reward-tilted noise distribution $p^\star$ that biases sampling toward noise realizations yielding high-reward outputs:
\begin{equation} \label{eqn:hypernoise_distribution}
p^\star(z_0) \propto p(z_0) \exp\big(r(G_\theta(z_0))\big),
\end{equation}
where $p$ denotes the original noise distribution. HyperNoise trains a noise hypernetwork $f_\phi$ with parameters $\phi$ such that the distribution of the modulated noise sample $z_0 + f_\phi(z_0)$, denoted by $p^\phi$, is close to the target distribution $p^\star$.  Accordingly, $f_\phi$ is trained to minimize the KL divergence $D_{\mathrm{KL}}(p^\phi\| p^\star)$, which can be approximated (up to a constant) under mild conditions by the following loss \cite{eyring2025noise}: 
\begin{equation}
    \mathcal{L}(\phi) = \mathbb{E}_{z_0 \sim p} \left[ \frac{1}{2} \| f_\phi(z_0) \|^2 - r(G_\theta(z_0+f_\phi(z_0))) \right], 
\end{equation}
which consists of a regularization term that penalizes deviation from the original noise distribution and a reward maximization term that encourages high-quality outputs. This formulation amortizes the search for high-reward noise into a learned mapping.
%enabling efficient inference-time optimization.

However, noise modulation remains computationally inefficient in both training and inference, since each forward pass requires executing a large-scale diffusion backbone. Furthermore, operating in the full latent noise space introduces unnecessary degrees of freedom, leading to reduced efficiency.

%However, noise modulation relies on the full diffusion backbone, requiring the execution of a large-scale model and incurring substantial computational cost. Furthermore, operating in the full latent noise space introduces unnecessary degrees of freedom, leading to reduced efficiency.

\paragraph{Patch-wise PCA Analysis of Diffusion Noise.}
Recently, an intriguing result has demonstrated a strong correlation between the PCA components of initial noise in diffusion models and the structure of the generated outputs \cite{wang2025seeds}. In this study, the noise is decomposed into $p \times p$ patches and projected onto a PCA basis derived from natural image statistics. It is shown that the top principal components (i.e., low-frequency components) largely determine key attributes of the generated images, such as global structure and object layout. In contrast, modifying high-frequency components has minimal impact on these structural properties. This observation motivates a more efficient noise modulation strategy that focuses solely on shaping the low-frequency components.

\section{Method}

We propose a novel framework, \textbf{Low-frequency Eigen Noise Shaping (LENS)}, for efficient diffusion sampling via noise modulation, which operates on the leading PCA subspace of the initial noise.\footnote{Following prior observations that leading PCA components of natural-image patches tend to exhibit smoother, lower-spatial-frequency patterns~\cite{lukoianov2025locality, hyvarinen2009natural}, we refer to this subspace as “low-frequency,” while not implying that the PCA basis itself is a Fourier basis. We further verify this interpretation for the PCA bases used by LENS in Appendix~\ref{low_frequency_exp}.} LENS employs a lightweight network operating on low-frequency noise coefficients to predict modulation, without relying on the diffusion backbone. It improves generation quality while maintaining low computational overhead in both training and inference.

\subsection{Theoretical Foundations and Objective} \label{theoretical analysis}

In the following, we introduce our method and provide theoretical justification based on prior findings that the low-frequency components of the initial noise predominantly determine the structural properties of the generated outputs \cite{wang2025seeds}. Building on the reward-tilted distribution and KL-based learning framework of HyperNoise~\cite{eyring2025noise}, we extend the formulation to noise modulation restricted to the low-frequency PCA subspace. We consider a pretrained diffusion generator $G_\theta$ that maps an initial noise sample $z_0 \in \mathbb{R}^{C \times H \times W}$ to an output image $x$, i.e., $x = G_\theta(z_0)$, where $C$, $H$, and $W$ denote the number of channels, height, and width, respectively. We assume $z_0\sim \mathcal{N}(0, I)$. Let $r(x)$ be a reward function that quantifies the quality of generated images. In practice, we also condition on a text prompt, and $r(x)$ reflects the alignment of $x$ with the given prompt. 

\paragraph{Noise Reparameterization in the PCA Coefficient Space.}
We partition the initial latent noise $z_0$ into $N$ non-overlapping patches $\{s_i\}_{i=1}^N$, where each patch $s_i \in \mathbb{R}^d$ with $d = C s^2$, is obtained by extracting a $s \times s$ spatial region across all channels and vectorizing it. Hence, $N = \frac{HW}{s^2}$. 

Each patch is projected onto a PCA basis $V \in \mathbb{R}^{d \times d}$ computed from natural image statistics, resulting in the coefficient vector $w_i = V^\top s_i \in \mathbb{R}^d$. Since $V$ is an orthonormal matrix, this transformation preserves the Gaussian prior, i.e., $w_i \sim \mathcal{N}(0, I)$. A detailed proof is provided in Appendix~\ref{proof_gaussian}. We further decompose the coefficient vector as $w_i = (w_{i,L}, w_{i,H})$, where $w_{i,L} \in \mathbb{R}^k$ denotes the low-frequency components and $w_{i,H} \in \mathbb{R}^{d-k}$ denotes the high-frequency components, with $k < d$. Let $w \in \mathbb{R}^{N \times d}$, $w_L \in \mathbb{R}^{N \times k}$, and $w_H \in \mathbb{R}^{N \times (d - k)}$ denote the collections $\{w_i\}_{i=1}^N$, $\{w_{i,L}\}_{i=1}^N$, and $\{w_{i,H}\}_{i=1}^N$, respectively. Then, their corresponding prior distributions, denoted by $q(w)$, $q_L(w_L)$, and $q_H(w_H)$, all follow isotropic Gaussian distributions $\mathcal{N}(0, I)$ with appropriate dimensions, as shown in Appendix~\ref{proof_gaussian2}.

\paragraph{Low-frequency-only Tilted Distribution.} Let $\textbf{Recon}(w)$ denote the reconstruction operator that maps patch-wise PCA coefficients to the full latent noise. The generator can be expressed in the coefficient space as $g(w) := G_\theta(\textbf{Recon}(w))$. Following the reward-tilted formulation of HyperNoise~\cite{eyring2025noise}, we define the target distribution in the PCA coefficient space as
\begin{equation}
q^\star(w) \propto q(w)\exp\big(r(g(w))\big).
\end{equation}

Prior empirical results from patch-wise PCA analysis \cite{wang2025seeds} show that perturbations applied to the low-frequency components of the noise have a dominant effect on the composition of generated images. Based on this observation, we make the following assumption, which is empirically validated in Section~\ref{experimental_results} and Appendix~\ref{assumption_1_validation}.
\begin{asmp} \label{assumption} For any $(w_L, w_H)$,
\begin{equation} 
\left| r\big(g(w_L, w_H)\big) - \bar{r}(w_L) \right| \leq \epsilon,
\end{equation}
where $\bar{r}(w_L) := \mathbb{E}_{w_H \sim q_H}[\,r(g(w_L, w_H))\,]$ denotes the expected reward conditioned on the low-frequency coefficients and $\epsilon > 0$ is a small constant that bounds the contribution of the high-frequency components to the reward.
\end{asmp}

Assumption~\ref{assumption} motivates restricting modulation to the low-frequency coefficients $w_L$ while leaving $w_H$ unchanged. The key challenge is to establish whether such a low-frequency-only restriction can still approximate the full reward-tilted target distribution $q^\star(w_L,w_H)$. Although the Gaussian prior factorizes as $q(w_L,w_H)=q_L(w_L)q_H(w_H)$, as shown in Appendix~\ref{proof_gaussian_factorization}, the reward $r(g(w_L,w_H))$ generally couples $w_L$ and $w_H$, and hence $q^\star(w_L,w_H)$ does not generally factorize. Therefore, leaving $w_H$ distributed according to its original prior $q_H$ is not, by itself, guaranteed to preserve the full target distribution. To justify this restriction, we define the low-frequency-only tilted distribution as
\begin{equation}
    \tilde{q}^\star(w) := \tilde{q}^\star_L(w_L)\, q_H(w_H), 
    \quad 
    \tilde{q}^\star_L(w_L) \propto q_L(w_L)\exp(\bar{r}(w_L)). 
\end{equation}
The following proposition shows that, under Assumption~\ref{assumption}, this low-frequency-only distribution remains close to the full target distribution $q^\star(w)$ in KL divergence. The proof is provided in Appendix~\ref{proof_proposition1}.

%We define the low-frequency-only tilted distribution as
%\begin{equation}
%    \tilde{q}^\star(w) := \tilde{q}^\star_L(w_L)\, q_H(w_H), 
%    \quad 
%    \tilde{q}^\star_L(w_L) \propto q_L(w_L)\exp(\bar{r}(w_L)). 
%\end{equation}
%The following proposition provides theoretical justification for approximating $q^\star(w)$ by $\tilde{q}^\star(w)$, whose proof is in Appendix \ref{proof_proposition1}. 

\begin{proposition} \label{proposition1}
The KL divergence between the full tilted distribution $q^\star(w)$ and its low-frequency-only tilted distribution $\tilde{q}^\star(w)$ is bounded as
\begin{equation}
    D_{\mathrm{KL}}(q^\star \parallel \tilde{q}^\star) \leq 2\epsilon.
\end{equation}
\end{proposition}

\paragraph{Objective.} 
We train a neural network $h_\phi$ with parameters $\phi$ that takes $w_L$ as input and outputs an additive perturbation to $w_L$. In practice, $h_\phi$ is also conditioned on a text prompt, but we omit this dependence for notational simplicity. Let $q_L^\phi$ denote the distribution of the modulated low-frequency coefficients $w_L + h_\phi(w_L)$. The network $h_\phi$ is trained so that $q_L^\phi$ approximates the target distribution $\tilde{q}_L^\star$. To this end, we minimize the KL divergence $D_{\mathrm{KL}}(q^{\phi}_L \| \tilde{q}^\star_L)$. The following proposition shows that this objective admits a tractable reformulation. The proof is provided in Appendix~\ref{proof_proposition2}.

%We train a neural network $h_\phi$ with parameters $\phi$. \hw{For notational simplicity, we write $h_\phi(w_L)$, while in practice the mapping is conditioned on a text prompt, producing coefficient modulations aligned with the prompt.} The distribution of the modulated low-frequency coefficients $w_L + h_\phi(w_L)$, denoted by $q_L^\phi$, approximates the target distribution $\tilde{q}_L^\star$. To this end, $h_\phi$ is trained to minimize the KL divergence $D_{\mathrm{KL}}(q^{\phi}_L \| \tilde{q}^\star_L)$. The following proposition shows that this objective can be made tractable. The proof is provided in Appendix \ref{proof_proposition2}. 

\begin{proposition}[Objective] \label{proposition2}
Assume that $h_\phi$ is continuously differentiable with $\|J_{h_\phi}(w_L)\|_2 \le M < 1$
for any $w_L$, where $J_{h_\phi}(w_L) := \frac{\partial h_\phi(w_L)}{\partial w_L}$ denotes the Jacobian matrix, $\|\cdot\|_2$ is the spectral norm, and $M$ is a constant.
Then minimizing the KL divergence $D_{\mathrm{KL}}(q^{\phi}_L \| \tilde{q}^\star_L)$ is equivalent (up to a constant) to minimizing 
\begin{equation}
    D_{\mathrm{KL}}(q_L^\phi \| q_L)
    - \mathbb{E}_{w_L \sim q_L^\phi}[\bar{r}(w_L)],
\end{equation} 
which can be approximated (up to $\epsilon$) by 
\begin{equation} \label{final_train_loss}
       \mathcal{L}(\phi):=\mathbb{E}_{w\sim q}\left[\frac{1}{2}\|h_\phi(w_L)\|^2-r \big(g(w_L+h_\phi(w_L),w_H)\big)\right].
\end{equation}

\end{proposition}

\begin{figure} 
    \centerline{\includegraphics[width=\columnwidth]{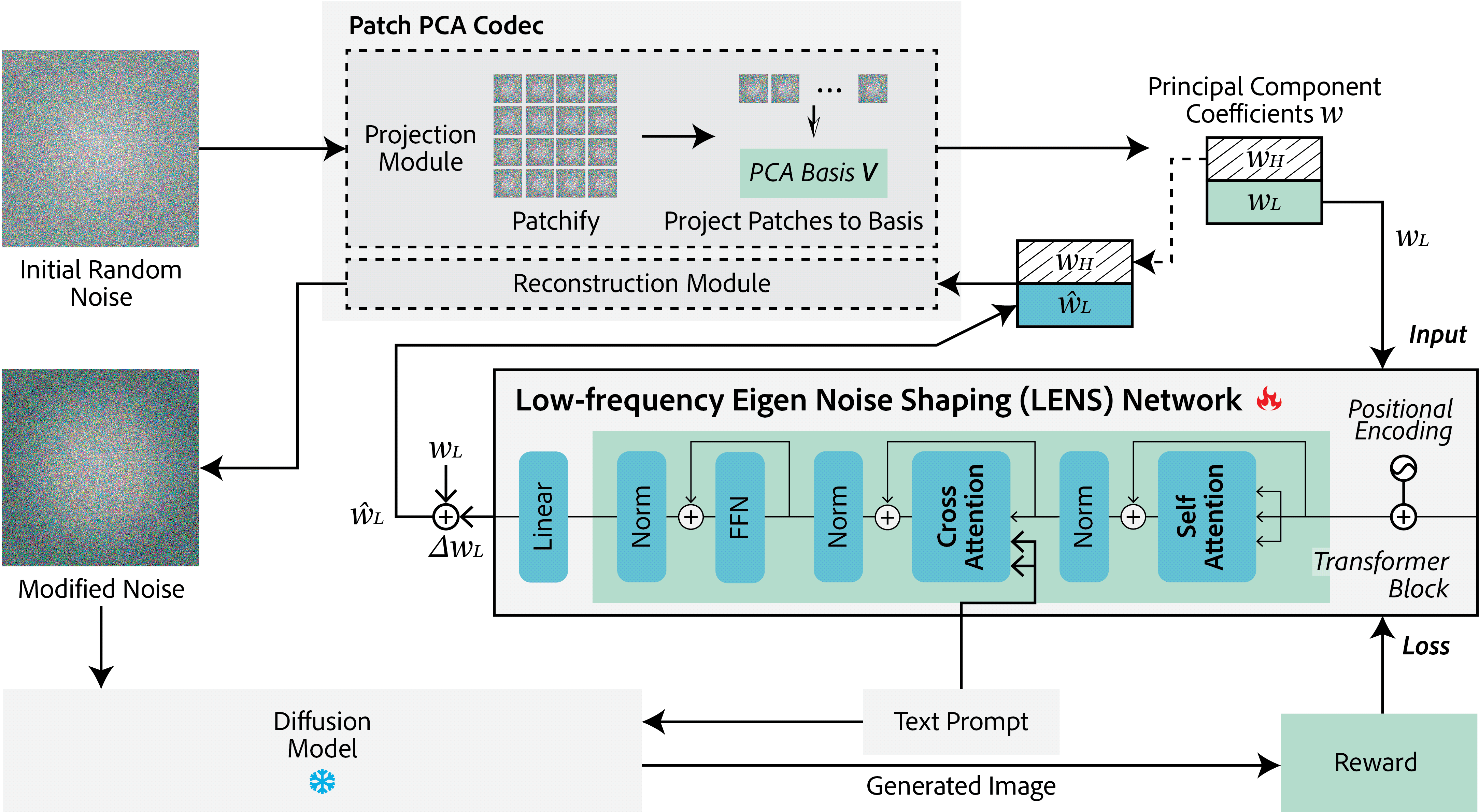}}
    \caption{Overview of the proposed \textbf{Low-frequency Eigen Noise Shaping (LENS)} framework. The model operates on low-frequency noise coefficients that capture structurally meaningful components. These coefficients are processed via self-attention, while the text prompt is incorporated through cross-attention for conditioning, and the network predicts coefficient updates.}
    \label{figure_LENS}
\end{figure}

\subsection{Implementation}
We present the practical implementation of the proposed framework. The framework comprises two components: a Patch PCA Codec and a LENS Network. The Patch PCA Codec projects noise onto a PCA basis to obtain coefficients and reconstructs them back to the noise space, while the LENS Network models prompt-conditioned modulation in the low-frequency coefficient space.

To construct the PCA basis, we precompute it from natural image datasets \cite{wang2025seeds}, such as ImageNetV2 \cite{recht2019imagenet}. Images are encoded into latent representations using a pretrained Variational Autoencoder (VAE)~\cite{kingma2014iclr-auto} of the diffusion model, and the resulting latents are partitioned into non-overlapping $s \times s$ patches. Each patch is reshaped into a vector in $\mathbb{R}^d$, and PCA is performed by computing the covariance matrix over all patches, yielding an orthonormal basis $V \in \mathbb{R}^{d \times d}$ whose columns correspond to eigenvectors.

The overall architecture and training pipeline are illustrated in Figure~\ref{figure_LENS}. We next describe the Patch PCA Codec and the LENS Network in detail.

\paragraph{Patch PCA Codec.}
This component consists of two modules: a projection module \textbf{Proj} and a reconstruction module \textbf{Recon}. The projection module partitions the input latent noise $z$ into non-overlapping patches and projects them onto the  precomputed PCA basis to obtain coefficient vectors $w$, i.e., $w=\textbf{Proj}(z)$. After coefficient modulation, the reconstruction module maps the coefficients $w$ back to the latent noise $z$ via the inverse PCA transformation, i.e., $z=\textbf{Recon}(w)$.

\paragraph{Low-frequency Eigen Noise Shaping (LENS) Network.}

The LENS network $h_\phi$ takes low-frequency coefficients $w_L$ and a text prompt $c$ as inputs, and predicts an additive perturbation to $w_L$ that aligns the generated output with the semantics of the prompt. Operating on low-frequency coefficients reduces the dimensionality of the prediction space, enabling a lightweight network decoupled from the diffusion backbone. We implement this network using a transformer architecture~\cite{vaswani2017attention} to model global dependencies and incorporate text conditioning. The network consists of an input projection layer, sinusoidal positional embeddings, multiple transformer blocks (including self-attention, text cross-attention, and feed-forward networks), and a final linear layer. Self-attention captures global correlations among coefficient tokens across patches, while text cross-attention aligns the modulation with the semantic content of the prompt. Further architectural details are provided in Appendix~\ref{architecture_explain}.

%\newpage
\begin{wrapfigure}{r}{0.5\textwidth}
\begin{minipage}{\linewidth}
\vspace{-24pt}
%arxiv
%\vspace{-12pt}
\begin{algorithm}[H]
\caption{Training of LENS Network}
\label{alg:training}
\begin{algorithmic}[1]
\State \textbf{Initialize:} LENS network parameters $\phi$

\While{learning}
    \State Sample latent noise $z_0 \sim \mathcal{N}(0, I)$
    \State $(w_L, w_H)=\textbf{Proj}(z_0)$
    \State $\hat{w} \gets (w_L + h_\phi(w_L,c),\, w_H)$
    %\State $\hat{z}_0=\textbf{Recon}(\hat{w})$
    \State $x \gets G_\theta(\textbf{Recon}(\hat{w}))$
    \State Compute reward $r(x,c)$ and loss $\mathcal{L}(\phi)$ 
    \State Update $\phi$ using gradient descent
\EndWhile
\State \textbf{return} trained parameters $\phi$

\end{algorithmic}
\end{algorithm}

\end{minipage}
\end{wrapfigure}
The network is pretrained and used once at inference time to modulate the low-frequency noise. The training procedure is summarized in Algorithm~\ref{alg:training}. We use Eq.~\eqref{final_train_loss} as the training loss $\mathcal{L}(\phi)$, which consists of a regularization term and a reward term. The former acts as an $L_2$ penalty on the coefficient residual, encouraging proximity to the original Gaussian prior. The latter is a weighted combination of multiple reward models, including CLIP~\cite{hessel2021clipscore}, HPSv2.1~\cite{wu2023human}, ImageReward~\cite{xu2023imagereward}, and PickScore~\cite{kirstain2023pick}. Additional details on the reward formulation are provided in Appendix~\ref{reward_explain}.

\subsection{Complexity Analysis}

We analyze the computational complexity of our transformer-based noise modulation framework, LENS. The computational complexity of a standard transformer is $\mathcal{O}(n^2 r + n r^2)$, where $n$ is the number of tokens and $r$ is the representation dimension~\cite{vaswani2017attention}. The first term corresponds to self-attention, while the second term accounts for the feed-forward network. 
In LENS, the number of tokens is $N$ (the number of patches), and each token has dimension $k$ (the number of low-frequency PCA coefficients). As a result, the computational complexity becomes $\mathcal{O}(N^2 k + N k^2)$. We omit the cost of cross-attention, which scales with the prompt length, as it is negligible compared to the self-attention term given that the prompt is significantly shorter than the number of tokens.

Compared to operating in the full latent dimension $d$, this reduces the self-attention and feed-forward complexities by factors of $d/k$ and $(d/k)^2$, respectively, yielding substantial computational savings. The efficiency of our method is experimentally validated in Section~\ref{sec:efficiency_experiment}.

%Our approach uses $k$-dimensional coefficients ($k < d$) instead of the full $d$-dimensional representation. Consequently, the computational cost scales with $k$. Reducing $k$ lowers the self-attention and feed-forward complexities by factors of $d/k$ and $(d/k)^2$, respectively, compared to using the full dimension $d$, resulting in significant computational savings. The efficiency of our method is experimentally validated in Section \ref{sec:efficiency_experiment}.

%Compared to methods operating on higher-dimensional representations (e.g., full latent noise) with transformer architectures, our approach uses $k$-dimensional coefficients ($k < d$). This reduces self-attention and feed-forward complexities by factors of $d/k$ and $(d/k)^2$, respectively, leading to significantly improved computational efficiency.

\section{Experiment}
\label{experiment}

\subsection{Experimental Settings}

\paragraph{Implementation.}
We conduct experiments using three distilled text-to-image diffusion models: SD-Turbo~\cite{sauer2024adversarial}, Hyper-SDXL \cite{ren2024hyper}, and SANA-Sprint~\cite{chen2025sana}. For each model, we precompute a corresponding PCA basis using its pretrained VAE encoder. We set the patch size to $s=4$. For SD-Turbo and Hyper-SDXL, the latent channel dimension is $C=4$, resulting in $d=64$, and we use $k=32$ low-frequency coefficients for modulation. For SANA-Sprint, the latent channel dimension is $C=32$, resulting in $d=512$, and we set $k=256$. The effects of $k$ and $s$ on performance are analyzed in Appendices~\ref{additional_exp_k} and \ref{additional_exp_s}, respectively. All experiments are conducted on NVIDIA RTX 3090 GPUs. Further implementation details are provided in Appendix~\ref{appendix_implementation_detail}.

\paragraph{Dataset.}
For PCA basis extraction, we use ImageNetV2~\cite{recht2019imagenet}. During training, we use approximately 70k text prompts collected from multiple sources, including Attribute Binding (ABC-6K)~\cite{feng2023trainingfree}, the T2I-CompBench training set~\cite{huang2023t2i}, and Pick-a-Pic v2~\cite{kirstain2023pick}.

\paragraph{Baselines.}
We compare our method against three baselines: random noise initialization (unmodified noise input), ReNo~\cite{eyring2024reno}, and HyperNoise~\cite{eyring2025noise}. Random noise initialization samples the initial latent from a standard Gaussian distribution without additional processing. ReNo performs reward-guided test-time optimization, iteratively updating the initial noise via gradient-based optimization to maximize a given reward function. HyperNoise trains a noise hypernetwork by attaching LoRA adapters to the diffusion backbone, using it to predict a modulated initial noise that is then fed into the generator.

%We compare our method against three baselines: random noise initialization, ReNo~\cite{eyring2024reno}, and HyperNoise~\cite{eyring2025noise}. Random noise initialization directly samples the initial latent from a standard Gaussian distribution without any additional processing. ReNo performs reward-guided test-time noise optimization by iteratively updating the initial noise for each sample via gradient-based optimization to maximize a given reward function. HyperNoise trains a noise hypernetwork, realized by attaching LoRA adapters to the full diffusion backbone, to predict a modulated initial noise, which is then used as input to the generator.

\begin{table}[t]
\caption{Results on GenEval2. Higher is better ($\uparrow$) for all metrics except Time. LENS achieves performance comparable to or better than HyperNoise while enabling faster generation. ReNo is shown in gray, as it relies on iterative test-time optimization, incurring substantially higher computational cost and making it not directly comparable to feed-forward methods. Standard deviations are reported in Appendix~\ref{appendix_std}.}
\vspace{5pt}
\centering
\resizebox{\textwidth}{!}{

\renewcommand{\arraystretch}{1.2}
\begin{tabular}{l c | ccccc |c}
\toprule
\textbf{Model} & \textbf{Time(s)} $\downarrow$ & \textbf{Object} $\uparrow$ & \textbf{Attribute} $\uparrow$ & \textbf{Count} $\uparrow$ & \textbf{Position} $\uparrow$ & \textbf{Verb} $\uparrow$ & \textbf{Mean} $\uparrow$ \\
\midrule
SD-Turbo & 0.107 & 72.27 & 42.65 & 21.15 & 28.1 & 5.60& 33.95 \\
\textcolor{lgray}{+ ReNo} & \textcolor{lgray}{11.371} & \textcolor{lgray}{81.02}& \textcolor{lgray}{50.42} & \textcolor{lgray}{31.91} & \textcolor{lgray}{32.58} & \textcolor{lgray}{6.90} & \textcolor{lgray}{40.57} \\
+ HyperNoise  & 0.194 & 74.94  & 43.00 & 23.95 & 27.67 & 5.23 & 34.96 \\
\textbf{+ LENS (Ours)} & \textbf{0.112} & \textbf{75.07} & \textbf{43.08} & \textbf{26.08} & \textbf{28.46} & \textbf{5.78} & \textbf{35.69} \\

\midrule
Hyper-SDXL & 0.414 & 71.94 & 51.97 & 26.88 & 21.14 & 3.28 & 35.04 \\
\textcolor{lgray}{+ ReNo} & \textcolor{lgray}{45.13} & \textcolor{lgray}{80.06}  & \textcolor{lgray}{56.79} & \textcolor{lgray}{32.61} & \textcolor{lgray}{27.53} & \textcolor{lgray}{3.59} & \textcolor{lgray}{40.12}  \\
+ HyperNoise & 0.694 & 71.97 & 52.02 & 27.06 & \textbf{21.58} & \textbf{4.76} & 35.48 \\
\textbf{+ LENS (Ours)} & \textbf{0.424} & \textbf{72.96} & \textbf{53.47} & \textbf{27.84} & 21.45 & 4.39 & \textbf{36.02} \\

\midrule
SANA-Sprint & 0.336 & 80.23 & 54.73 & 33.01 & 33.31 & 15.53 & 43.36 \\
\textcolor{lgray}{+ ReNo} & \textcolor{lgray}{14.964} & \textcolor{lgray}{83.16}& \textcolor{lgray}{56.57} & \textcolor{lgray}{34.73} & \textcolor{lgray}{38.36} & \textcolor{lgray}{13.44} & \textcolor{lgray}{45.25} \\
+ HyperNoise & 0.532 & \textbf{83.59} & 54.18 & 34.84 & 36.32 & \textbf{20.66} & \textbf{45.92}\\
\textbf{+ LENS (Ours)} & \textbf{0.352} & 82.98 & \textbf{55.27} & \textbf{36.12} & \textbf{36.38} & 18.21 & 45.79 \\

\bottomrule
\end{tabular}
}

\label{tab:geneval_results}

\end{table}

\subsection{Experimental Results} \label{experimental_results}
\paragraph{Generation Quality.}
We evaluate the quality of generated images on the GenEval2 benchmark~\cite{kamath2025geneval}, with results reported in Table~\ref{tab:geneval_results}. All images are generated using a single diffusion step. Our method, LENS, achieves comparable or better performance than HyperNoise while reducing the additional time cost by a factor of 10--20. More specifically, HyperNoise incurs additional computational cost of approximately 50\%--80\% of the image generation cost, as the diffusion backbone must be executed once more during noise modulation. In contrast, our method introduces less than 5\% overhead. Although ReNo attains the best performance in many cases, it requires  40--100$\times$ longer generation time than our method, making it impractical for real-world applications. We analyze performance as the number of inference steps increases in Appendix~\ref{additional_exp_multi}. Additional results on other benchmarks are provided in Appendix~\ref{additional_exp_other}.

%We report experimental results on the GenEval2 benchmark~\cite{kamath2025geneval} in Table~\ref{tab:geneval_results}. All images are generated using a single diffusion step. Our method, LENS, achieves comparable or better performance than the baselines while generating images at a significantly faster speed. Although ReNo attains the best performance, it requires more than 100$\times$ longer generation time compared to our method, making it impractical for real-world applications. HyperNoise incurs additional computational cost equivalent to 0.5--1$\times$ that of image generation, as the diffusion backbone is executed once more during noise modulation. In contrast, our method introduces under 5\% overhead. 

\begin{table}[t] 

\caption{Noise modulation efficiency comparison between LENS and HyperNoise across different diffusion models. LENS consistently achieves lower FLOPs, fewer parameters, and faster inference time, resulting in significantly improved computational efficiency.}
\vspace{5pt}
\centering

\resizebox{\columnwidth}{!}{
{\fontsize{4.5}{5}\selectfont
\renewcommand{\arraystretch}{0.55}

\setlength{\heavyrulewidth}{0.09em}
\setlength{\lightrulewidth}{0.06em}
\begin{tabular}{llccc}
\toprule 
\textbf{Model} & \textbf{Method} & \textbf{FLOPs} & \textbf{Parameters}  & \textbf{Inference-time} \\
\midrule
\multirow{2}{*}{SD-Turbo}    & HyperNoise & 997.1 G & 136.5 M & 70.78 ms\\
\addlinespace[1pt]
& \textbf{LENS (Ours)}  & 2.4 G & 4.8 M & 5.68 ms  \\

\midrule
\multirow{2}{*}{Hyper-SDXL}  & HyperNoise  & 7.8 T & 395.3 M &  207.91 ms \\
\addlinespace[1pt]
& \textbf{LENS (Ours)}  & 12.3 G & 5.3 M & 9.83 ms  \\

\midrule
\multirow{2}{*}{SANA-Sprint} & HyperNoise & 1.7 T & 314.7 M & 173.62 ms \\
\addlinespace[1pt]
& \textbf{LENS (Ours)}  & 2.2 G & 10.0 M & 10.84 ms  \\

\bottomrule
\end{tabular}
}
}

\label{tab:efficiency_comparison}

\end{table}

\paragraph{Computational Efficiency.} \label{sec:efficiency_experiment}
In Table~\ref{tab:efficiency_comparison}, we evaluate the computational efficiency of our method and HyperNoise~\cite{eyring2025noise}, a representative approach that learns a noise modulation network. We report FLOPs for the noise-modulation pathway only, excluding full image generation. For HyperNoise, this includes a forward pass through the diffusion backbone (including LoRA branches), whereas for LENS, it includes PCA projection/reconstruction and a forward pass of the LENS network. We also report the number of trainable parameters in the modulation network, counting LoRA parameters for HyperNoise and all parameters of the LENS network for our method. Inference time is measured as the time required for noise modulation, averaged over 100 samples. Our method achieves substantial efficiency gains, reducing FLOPs by 400–700$\times$, the number of parameters by 25–75$\times$, and inference-time latency by 10–20$\times$ across all evaluated models.

%Our method achieves substantial efficiency gains, with over 400$\times$ reduction in FLOPs, 28$\times$ fewer parameters, and more than 12$\times$ faster inference-time execution across all evaluated models.
% To evaluate computational efficiency, we compare our method with HyperNoise~\cite{eyring2025noise}, a representative approach that operates in the full noise space and relies on a diffusion backbone. We report FLOPs only for the noise-modulation pathway, excluding full image generation. For HyperNoise, this includes the full noise-modulation forward through the diffusion backbone (including LoRA branches), whereas for LENS, it includes PCA projection/reconstruction and the LENS network forward. We report the number of parameters in the modulation network, counting LoRA parameters for HyperNoise and the full LENS network for our method. The inference time is measured as the time required to modulate the noise, averaged over 100 samples. We also report the training step time, defined as the average per-step time during the second epoch to avoid initialization overhead. This includes latent sampling, modulation, image generation, reward computation, and parameter updates.
% Both inference time and training step time are measured on the same GPU for each model, and the step time is reported under the same batch size. As shown in Table \ref{tab:efficiency_comparison}, our method achieves at least $100\times$ lower FLOPs, $80\times$ fewer parameters, and around $2\times$ faster inference-time step across all evaluated models.

\paragraph{Low-Frequency Concentration of Reward Gradients.}

\begin{figure}[t]
    
    \centering
    \begin{subfigure}{0.49\textwidth}
        \centering
        \includegraphics[width=\linewidth]{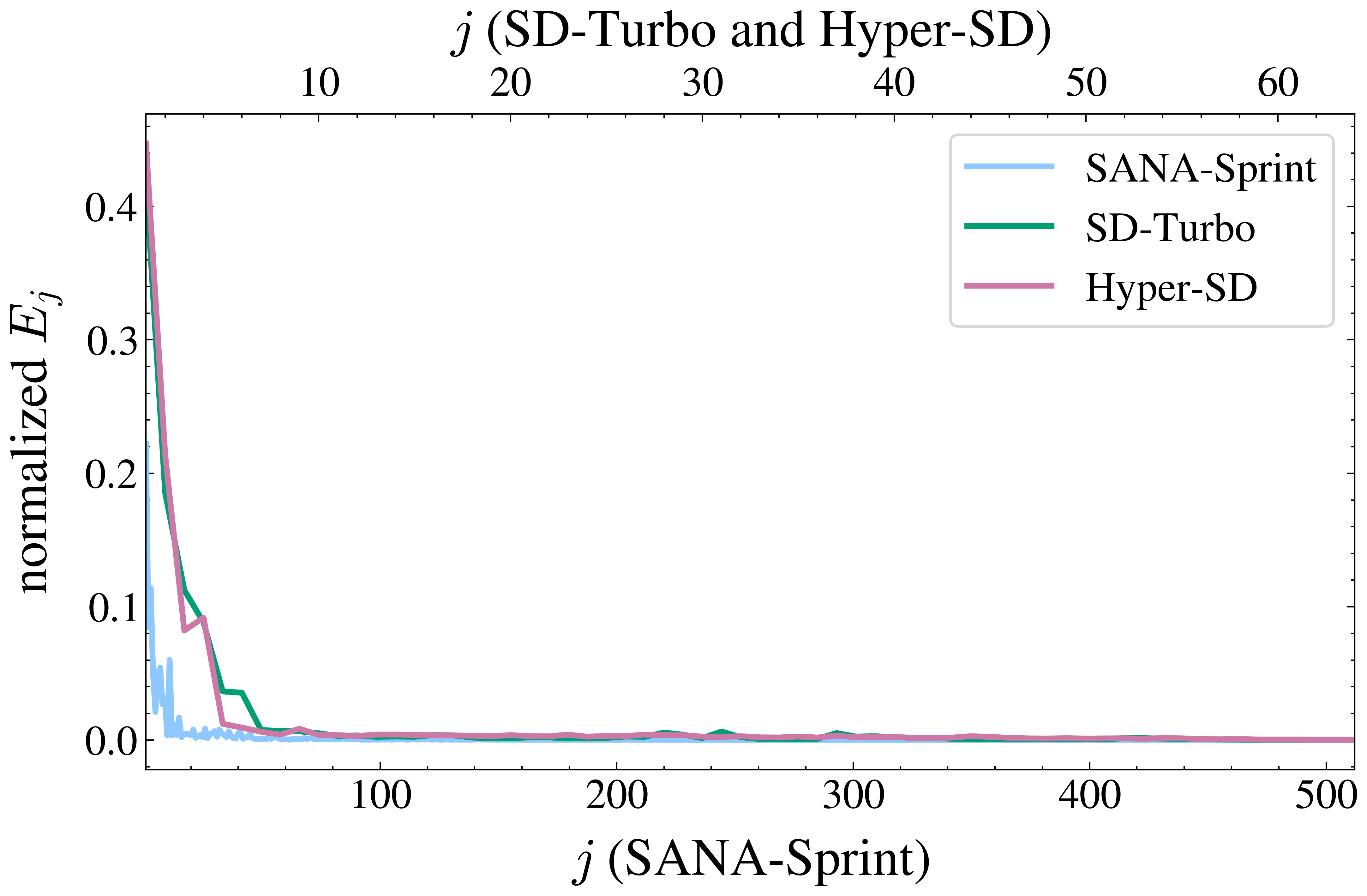}
        \caption{Reward gradient energy spectrum in the PCA basis}
    \end{subfigure}
    \hfill
    \begin{subfigure}{0.49\textwidth}
        \centering
        \includegraphics[width=\linewidth]{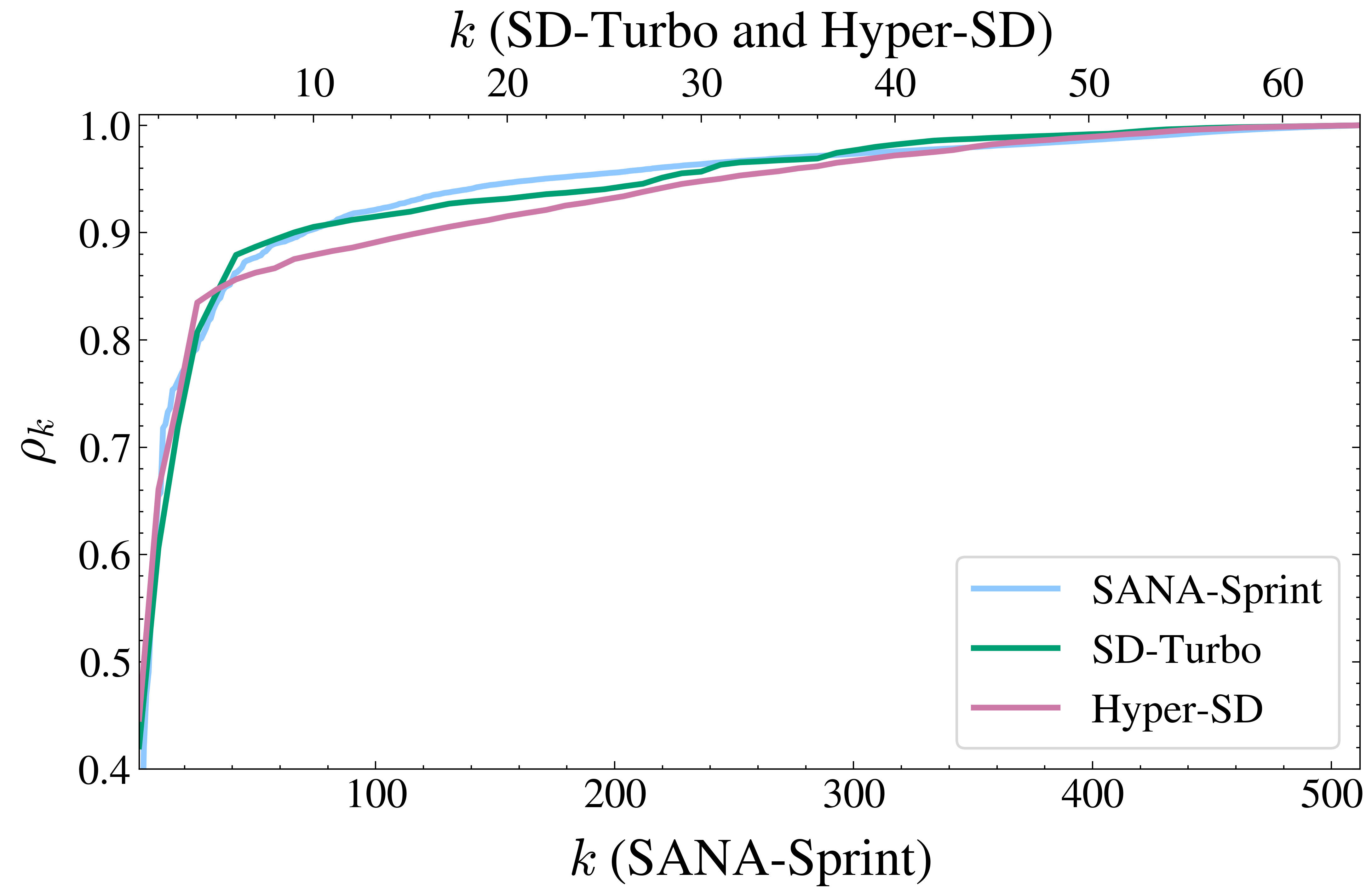}
        \caption{Cumulative gradient energy ratio}
    \end{subfigure}
 
    \caption{Reward gradient energy distribution in the PCA basis. 
(a) Normalized energy spectrum $E_j / \sum_j E_j$, showing a rapid decay toward higher-frequency components.
(b) Cumulative energy ratio $\rho_k$, indicating that most energy is concentrated in low-frequency components.}
    \label{fig:reward_gradient}
    \vspace{-3pt}
\end{figure}

In this experiment, we validate the key hypothesis that reward-increasing directions are concentrated in a low-frequency PCA subspace rather than the full high-dimensional noise space. To this end, given an initial noise $z_0$ and a prompt $c$, we compute the reward gradient $\nabla_{z_0} r\bigl(G_\theta(z_0, c)\bigr)$, which indicates the direction in which the noise should be perturbed to maximally increase the reward. Then, we reshape the reward gradient into patch-wise vectors $\{u_\ell\}$ and project each onto the PCA basis $V = [v_1, \dots, v_d] \in \mathbb{R}^{d \times d}$, where $v_j$ denotes the $j$-th eigenvector, yielding coefficients $\alpha_{j,\ell} = v_j^\top u_\ell$. The importance of each direction is quantified by the expected gradient energy $E_j = \mathbb{E}_{\ell}[\alpha_{j,\ell}^2]$, where the expectation is taken over spatial patches and 100 sampled images. The cumulative energy ratio captured by the top-$k$ components is then defined as $\rho_k = \frac{\sum_{j=1}^{k} E_j}{\sum_{j=1}^{d} E_j}$. 

Figure~\ref{fig:reward_gradient}(a) shows that the normalized $E_j$ values decay rapidly as $j$ increases, indicating that high-frequency components contribute little to the overall gradient energy. Figure~\ref{fig:reward_gradient}(b) further shows that $\rho_k$ quickly approaches a high value even for small $k$. These results confirm that the reward gradient is largely concentrated in a low-dimensional PCA subspace.

\begin{table}[t]

\caption{Generation diversity measured by LPIPS and DINO. LENS achieves intermediate LPIPS between random noise and HyperNoise, while attaining the lowest DINO score among all methods. Values are reported as mean $\pm$ standard deviation.}
\vspace{5pt}
\centering
\renewcommand{\arraystretch}{1.4} 
\setlength{\tabcolsep}{18pt}   
\fontsize{10}{10}\selectfont
\begin{tabular}{lcc}
\toprule 
\textbf{Method}  & \textbf{LPIPS $\uparrow$} & \textbf{DINO $\downarrow$} \\
\midrule
SD-Turbo  & \textbf{0.6249 $\pm$ 0.0427} & 0.7531 $\pm$ 0.1107  \\
+ HyperNoise & 0.5406 $\pm$ 0.0562 & 0.8120 $\pm$ 0.1047  \\
\textbf{+ LENS (Ours)}  & 0.5887 $\pm$ 0.0404 & \textbf{0.7337 $\pm$ 0.1221}  \\
\bottomrule
\end{tabular}
\label{tab:diversity}

\end{table}

\paragraph{Generation Diversity.}
To evaluate generation diversity, we generate images for 300 prompts from the GenEval2 benchmark using 50 different random noise seeds, and compute LPIPS~\cite{zhang2018unreasonable} and DINOv2~\cite{oquab2023dinov2} scores. LPIPS measures perceptual differences between images, where higher values indicate greater visual diversity. DINO, on the other hand, captures feature-level similarity, where lower scores correspond to lower semantic similarity and thus higher semantic diversity.

Table~\ref{tab:diversity} reports LPIPS and DINOv2 scores for unmodified noise, HyperNoise, and LENS. The results show that LENS achieves lower LPIPS than unmodified noise, indicating a moderate reduction in pixel-level diversity, while still maintaining higher diversity than HyperNoise. Under the DINO metric, LENS attains lower scores than unmodified noise, suggesting improved semantic diversity. These results demonstrate that LENS preserves a favorable level of diversity while enhancing semantic variation.

\begin{table}[t]
\caption{Out-of-distribution prompt evaluation on OneIG-Bench. LENS achieves the best performance on both Style and Reasoning subsets without any retraining or adaptation.}
\vspace{5pt}
\centering
\renewcommand{\arraystretch}{1.4}
\setlength{\tabcolsep}{20pt}   
\fontsize{10}{10}\selectfont

\label{tab:ood_oneig}
\begin{tabular}{lcc}
\toprule
\textbf{Method} & \textbf{Style $\uparrow$} & \textbf{Reasoning $\uparrow$} \\
\midrule
SD-Turbo     & 0.3360 & 0.1907 \\
+ HyperNoise & 0.3190 & 0.1862 \\
\textbf{+ LENS (Ours)}       & \textbf{0.3474} & \textbf{0.1946} \\
\bottomrule
\end{tabular}
\vspace{-3pt}
\end{table}

\paragraph{Out-of-Distribution Prompt Evaluation.}
To evaluate the out-of-distribution (OOD) generalization of LENS, we additionally test on the Anime  \& Stylization and Knowledge \& Reasoning subsets of OneIG-Bench~\cite{chang2025oneig}, which contain prompts substantially different from those used to train the modulation network. Using SD-Turbo, we compare unmodified noise, HyperNoise, and LENS under the standard OneIG-Bench evaluation protocol~\cite{chang2025oneig}. In this experiment, both the PCA basis and the trained modulation network are applied without any retraining or adaptation.

As shown in Table~\ref{tab:ood_oneig}, LENS achieves the highest scores on both the Style and Reasoning subsets, outperforming unmodified noise and HyperNoise. These results suggest that the learned modulation generalizes well to the evaluated OOD prompts and remains effective even when applied to prompt categories that differ substantially from those seen during training. This further indicates that the benefit of LENS is not limited to the training prompt distribution.

\begin{figure} [t]

    \centerline{\includegraphics[width=0.95\columnwidth]{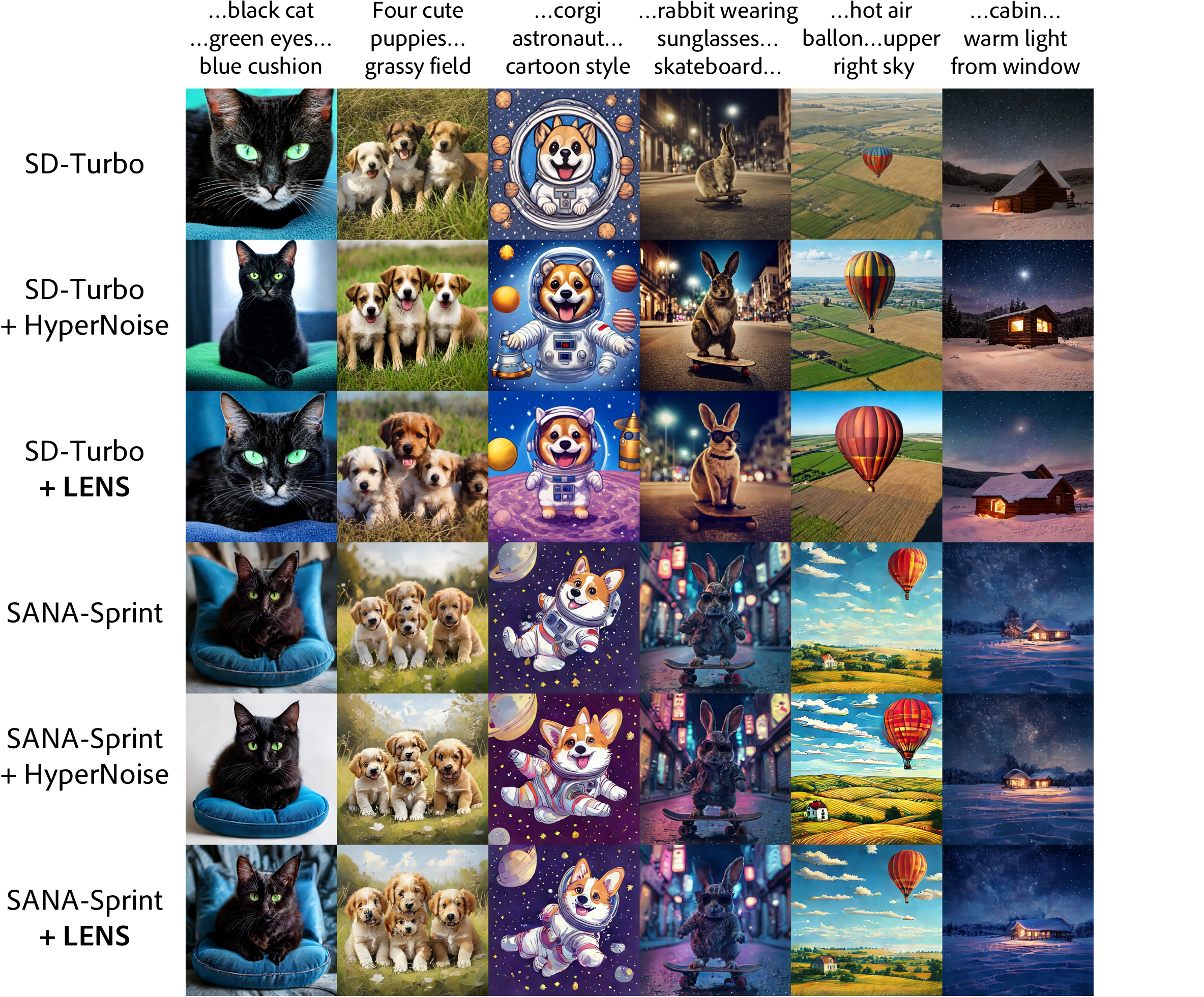}}

    \caption{One-step generation results using SD-Turbo and SANA-Sprint. Starting from the same initial noise, we compare outputs generated with unmodified noise, HyperNoise, and LENS.}
    \label{fig_images}

\end{figure}

\paragraph{Qualitative Results.}
Figure~\ref{fig_images} illustrates one-step generation results from SD-Turbo and SANA-Sprint. All samples for each prompt are generated from the same initial noise to ensure a fair comparison. Without noise modification, the generated images often exhibit poor alignment with the given prompts. In contrast, both HyperNoise and LENS improve semantic consistency and overall visual quality. However, HyperNoise occasionally exhibits inconsistencies with key prompt attributes, whereas LENS generates images that remain more faithful to the intended semantics. LENS further enhances structural clarity and overall generation quality, demonstrating more reliable high-quality generation in the one-step setting.

\paragraph{More Experiments.}

We provide additional experiments in Appendix~\ref{additional_exp}. Specifically, Appendices~\ref{additional_exp_k} and \ref{additional_exp_s} present ablation studies on the number of low-frequency coefficients $k$ and the patch size $s$, justifying our design choices.  Appendix~\ref{additional_exp_layer} examines the effect of the number of transformer layers in the LENS network, showing that increasing model depth does not necessarily improve performance. Appendix~\ref{additional_exp_multi} analyzes performance as the number of inference steps increases, complementing the one-step focus of the main paper, and shows that the distilled models are optimized for few-step generation. Appendix~\ref{additional_exp_other} reports results on additional benchmarks, confirming the robustness of our method across different settings. Appendix~\ref{asthetic_eval} evaluates aesthetic quality, demonstrating favorable overall visual quality. Appendices~\ref{reward_ablation} and~\ref{pca_data_ablation} provide ablations on reward components and PCA dataset sensitivity, showing that the full reward combination performs best and that LENS remains effective across different PCA basis datasets. Finally, standard deviations for the results in Table \ref{tab:geneval_results} are provided in Appendix~\ref{appendix_std}.

\section{Conclusion}
In this paper, we proposed LENS, a principled and efficient framework for noise modulation in diffusion models. By leveraging the insight that low-frequency components of the noise largely determine the global structure of generated images, we restrict modulation to a low-dimensional subspace and derive a corresponding training objective. This design enables effective improvement in generation quality while significantly reducing computational overhead. 
Extensive experiments across multiple distilled diffusion backbones demonstrate that LENS consistently achieves competitive or improved image quality with minimal inference-time overhead. These results establish low-dimensional noise shaping as an effective approach for improving diffusion-based generation, offering a practical pathway toward more efficient and scalable generative models. A discussion of the broader impact is provided in Appendix~\ref{broader}.

\paragraph{Limitations and Future Work.}
Our approach relies on a fixed PCA basis computed from a reference dataset, which may not fully capture data distributions that differ significantly from the training domain. Adapting the basis to new domains or learning it jointly with the modulation network could further improve generalization. In addition, LENS focuses on low-frequency components of the noise, which are primarily responsible for global structure. Although this design enables efficient modulation, it may limit the ability to refine fine-grained details governed by high-frequency components. Extending the framework to selectively incorporate higher-frequency information in a controlled manner is an interesting direction for future work.

\newpage

\begin{ack}

This work was supported in part by the IITP (Institute of Information \& Communications Technology Planning \& Evaluation)-ITRC (Information Technology Research Center) grant funded by the Korea government (MSIT) (IITP-2026-RS-2020-II201787, contribution rate 10\%), in part by the IITP through the Next Generation Semantic Communication Network Research Center grant funded by Korean Government (MSIT) (RS-2024-00398948, contribution rate 20\%), and in part by the IITP grant funded by the Korea Media and Communications Commission (KMCC) (No. RS-2026-25514346, Persona AI Model Technology Based on Accumulated Media Consumption Histories, contribution rate 70\%).

\end{ack}

\bibliographystyle{plainnat}
\bibliography{references}

%%%%%%%%%%%%%%%%%%%%%%%%%%%%%%%%%%%%%%%%%%%%%%%%%%%%%%%%%%%%
\newpage
\appendix

%\section*{Appendix}

\section{Theoretical Derivations}
\subsection{Preliminaries}

\paragraph{Gaussian Distribution.}
We denote $\mathcal{N}(\mu, \Sigma)$ as a multivariate Gaussian distribution with mean $\mu$ and covariance matrix $\Sigma$. In particular, let $\phi_d(x)$ be the probability density function (PDF) of the $d$-dimensional standard Gaussian distribution $\mathcal{N}(0, I_d)$, defined as:
\begin{equation}
    \phi_d(x) = \frac{1}{(2\pi)^{d/2}} \exp\left( -\frac{1}{2} \|x\|^2 \right).
\end{equation}

A fundamental property of Gaussian distributions is their closure under affine transformations. Specifically, if a random vector $x$ follows a Gaussian distribution $x \sim \mathcal{N}(\mu, \Sigma)$, then for any matrix $A \in \mathbb{R}^{m \times d}$ and vector $b \in \mathbb{R}^m$, the transformed variable $y = Ax + b$ is also Gaussian:
\begin{equation}
    y \sim \mathcal{N}(A\mu + b, A\Sigma A^\top).
\end{equation}

Furthermore, the standard Gaussian distribution $\mathcal{N}(0, I_d)$ exhibits \textit{rotational invariance}. If $A$ is an orthogonal matrix (i.e., $A^\top A = AA^\top = I_d$), the covariance structure is preserved under the transformation $y = Ax$, yielding:
\begin{equation}
    y \sim \mathcal{N}(0, I_d).
\end{equation}
This property ensures that the statistical characteristics of the Gaussian distribution remain unchanged under orthogonal transformations, such as rotations or changes of basis.

\paragraph{Exponential Tilting.}
Given a base distribution $q_0(x)$ and a real-valued function $f(x)$, the exponentially tilted distribution is defined as
\begin{equation}
    q(x) = \frac{1}{Z} q_0(x)\exp(f(x)),
\end{equation}
where $Z = \int q_0(x)\exp(f(x))\,dx$ is the normalization constant, assumed to be finite.

This construction reweights the base distribution $q_0$ by assigning higher probability mass to regions where $f(x)$ is large. As a result, $q(x)$ is biased toward such regions while preserving the overall structure of the base distribution.

In the context of reward-based modeling, $f(x)$ typically corresponds to a reward function $r(x)$, and the tilted distribution takes the form
\begin{equation}
    q(x) \propto q_0(x)\exp(r(x)).
\end{equation}
This formulation commonly arises in probabilistic inference and energy-based models, where the exponential term softly biases the distribution toward high-reward regions while maintaining stochasticity.

Importantly, exponential tilting only reweights the base distribution $q_0$ and does not assign probability mass to regions where $q_0(x) = 0$. In other words, if a point $x$ is impossible under $q_0$, it remains impossible under $q$. As a result, the support of $q$ is contained within that of $q_0$. This means that exponential tilting does not introduce new regions in the space, but instead redistributes probability mass within the existing domain. Such a property makes it particularly suitable for guiding generation, as it biases the distribution toward desirable regions while preserving the overall structure of the original distribution.

\paragraph{Kullback--Leibler Divergence.}
For two probability densities $p$ and $q$, the Kullback--Leibler (KL) divergence is defined as
\begin{equation}
    D_{\mathrm{KL}}(p \,\|\, q)
    =
    \mathbb{E}_{p}\left[\log \frac{p(x)}{q(x)}\right],
\end{equation}
assuming that $p$ is absolutely continuous with respect to $q$. It measures the discrepancy between two distributions, although it is not a symmetric distance. The KL divergence is always non-negative and equals zero if and only if $p = q$ almost everywhere.

An equivalent expression is
\begin{equation}
    D_{\mathrm{KL}}(p \,\|\, q)
    =
    \mathbb{E}_{p}[\log p(x)] - \mathbb{E}_{p}[\log q(x)],
\end{equation}
which shows that minimizing $D_{\mathrm{KL}}(p \,\|\, q)$ encourages $q$ to assign high probability to samples drawn from $p$.

\subsection{Gaussian Prior Preservation under PCA Projection} \label{proof_gaussian}
We show that the orthonormal PCA transformation preserves the Gaussian distribution.
Let $z_0 \sim \mathcal{N}(0, I_D)$ and partition it into non-overlapping patches $\{s_i\}_{i=1}^N$, where $s_i \in \mathbb{R}^d$.
Define $w_i = V^\top s_i$ with an orthonormal matrix $V \in \mathbb{R}^{d \times d}$.
Then $w_i \sim \mathcal{N}(0, I_d)$.

\begin{proof}
Let $B_i \in \mathbb{R}^{d \times D}$ denote the selection matrix that extracts the $i$-th patch from $z_0$, so that
\begin{equation}
    s_i = B_i z_0.
\end{equation}
Since $z_0 \sim \mathcal{N}(0, I_D)$ and $s_i$ is a linear transformation of $z_0$, $s_i$ is also Gaussian:
\begin{equation}
    s_i \sim \mathcal{N}(B_i 0,\; B_i I_D B_i^\top).
\end{equation}
Because $B_i$ selects $d$ distinct coordinates, we have $B_i I_D B_i^\top = B_i B_i^\top = I_d$.
Hence,
\begin{equation}
    s_i \sim \mathcal{N}(0, I_d).
\end{equation}

Applying the property of Gaussian distribution to $s_i \sim \mathcal{N}(0, I_d)$ with $A = V^\top$ and define $w_i = V^\top s_i$, we obtain
\[
w_i \sim \mathcal{N}(V^\top 0,\; V^\top I_d V).
\]
Using the orthonormality of $V$, $V^\top I_d V = V^\top V = I_d$.

Therefore,
\begin{equation}
    w_i \sim \mathcal{N}(0, I_d).
\end{equation}

\end{proof}

\subsection{Gaussian Prior Preservation for the Full Coefficient Collection}  \label{proof_gaussian2}

We show that the full stacked coefficient vector, obtained from non-overlapping patches via an orthonormal PCA transform, follows a standard Gaussian distribution.

Let
\begin{equation}
    w
    :=
    \begin{bmatrix}
        w_1 \\
        w_2 \\
        \vdots \\
        w_N
    \end{bmatrix}
    \in \mathbb{R}^{Nd},
\end{equation}
where each $w_i \in \mathbb{R}^d$ denotes the PCA coefficient vector of the $i$-th patch. For notational convenience, $w$ may also be viewed as an $N \times d$ array, but we treat it as a vector in $\mathbb{R}^{Nd}$ when defining its distribution.

We will show that
\begin{equation}
    w \sim \mathcal{N}(0, I_{Nd}),
\end{equation}
which we denote by $q(w) = \mathcal{N}(0, I)$.

\begin{proof}
From Section~\ref{proof_gaussian}, for each patch $i$ we have
\begin{equation}
    w_i = V^\top s_i \sim \mathcal{N}(0, I_d),
\end{equation}
where $V$ is an orthonormal PCA basis matrix satisfying $V^\top V = VV^\top = I_d$. To establish the joint distribution of the full collection $w$, it remains to analyze the stacked patch vector.

Define
\begin{equation}
    \mathbf{s} :=
    \begin{bmatrix}
        s_1 \\
        s_2 \\
        \vdots \\
        s_N
    \end{bmatrix}
    \in \mathbb{R}^{Nd}.
\end{equation}
Each patch is obtained from the initial noise $z_0 \sim \mathcal{N}(0, I_D)$ via a linear selection:
\begin{equation}
    s_i = B_i z_0,
\end{equation}
where $B_i \in \mathbb{R}^{d \times D}$ is a selection matrix that extracts the coordinates corresponding to the $i$-th non-overlapping patch.

Stacking all patches, define
\begin{equation}
    B :=
    \begin{bmatrix}
        B_1 \\
        B_2 \\
        \vdots \\
        B_N
    \end{bmatrix}
    \in \mathbb{R}^{Nd \times D},
\end{equation}
so that
\begin{equation}
    \mathbf{s} = B z_0.
\end{equation}
Since $z_0 \sim \mathcal{N}(0, I_D)$ and $p$ is a linear transformation of $z_0$, it follows that
\begin{equation}
    \mathbf{s} \sim \mathcal{N}(0, B I_D B^\top) = \mathcal{N}(0, BB^\top).
\end{equation}

Because the patches are non-overlapping and each $E_i$ selects $d$ distinct coordinates, we have
\begin{equation}
    B_i B_i^\top = I_d, \quad
    B_i B_j^\top = 0 \quad (i \neq j).
\end{equation}
Thus,
\begin{equation}
    BB^\top
    =
    \begin{bmatrix}
        I_d & 0   & \cdots & 0 \\
        0   & I_d & \cdots & 0 \\
        \vdots & \vdots & \ddots & \vdots \\
        0   & 0   & \cdots & I_d
    \end{bmatrix}
    = I_{Nd},
\end{equation}
which implies
\begin{equation}
    \mathbf{s} \sim \mathcal{N}(0, I_{Nd}).
\end{equation}

Next, define the block-diagonal matrix
\begin{equation}
    \bar{V}
    :=
    \begin{bmatrix}
        V^\top & 0      & \cdots & 0 \\
        0      & V^\top & \cdots & 0 \\
        \vdots & \vdots & \ddots & \vdots \\
        0      & 0      & \cdots & V^\top
    \end{bmatrix}
    \in \mathbb{R}^{Nd \times Nd}.
\end{equation}
Then the full coefficient vector can be written as
\begin{equation}
    w = \bar{V} \mathbf{s}.
\end{equation}

Since $\mathbf{s} \sim \mathcal{N}(0, I_{Nd})$ and $\bar{V}$ is orthonormal (as each block $V^\top$ is orthonormal), we have
\begin{equation}
    w
    \sim
    \mathcal{N}(0, \bar{V} I_{Nd} \bar{V}^\top)
    =
    \mathcal{N}(0, \bar{V}\bar{V}^\top)
    =
    \mathcal{N}(0, I_{Nd}).
\end{equation}

Therefore, the full stacked coefficient vector follows a standard Gaussian distribution:
\begin{equation}
    w \sim \mathcal{N}(0, I_{Nd}).
\end{equation}
This completes the proof.
\end{proof}

\subsection{Factorization of the Gaussian Prior over Low- and High-Frequency Components}
\label{proof_gaussian_factorization}

We show that the standard Gaussian prior factorizes over the low- and high-frequency components.

Let
\begin{equation}
    w
    =
    \begin{bmatrix}
        w_1 \\
        w_2 \\
        \vdots \\
        w_N
    \end{bmatrix}
    \in \mathbb{R}^{N \times d},
\end{equation}
where each $w_i \in \mathbb{R}^d$, and decompose
\begin{equation}
    w = (w_L, w_H),
\end{equation}
with $w_L \in \mathbb{R}^{N \times k}$ and $w_H \in \mathbb{R}^{N \times (d-k)}$.

Then the Gaussian prior factorizes as
\begin{equation}
    q(w) = q_L(w_L)\, q_H(w_H).
\end{equation}

\begin{proof}
From Appendix~\ref{proof_gaussian2}, the full coefficient collection satisfies
\begin{equation}
    w \sim \mathcal{N}(0, I).
\end{equation}
Under the decomposition $w = (w_L, w_H)$, the covariance matrix can be written in block form as
\begin{equation}
    \mathrm{Cov}(w)
    =
    \begin{bmatrix}
        \mathrm{Cov}(w_L) & \mathrm{Cov}(w_L,w_H) \\
        \mathrm{Cov}(w_H,w_L) & \mathrm{Cov}(w_H)
    \end{bmatrix}.
\end{equation}
Since $w$ is standard Gaussian, this covariance equals the identity, implying
\begin{equation}
    \mathrm{Cov}(w_L,w_H)=0.
\end{equation}
Because $(w_L,w_H)$ is jointly Gaussian with zero cross-covariance, the two variables are independent. Therefore,
\begin{equation}
    q(w_L,w_H)
    =
    q_L(w_L)\, q_H(w_H).
\end{equation}
\end{proof}

\subsection{(Proof of Proposition \ref{proposition1}) Low-Frequency Approximation of the Tilted Distribution} \label{proof_proposition1}
\begin{proof}
By definition of the reward-tilted distribution,
\begin{equation}
    q^\star(w)\propto q(w)\exp\big(r(g(w_L,w_H))\big).
\end{equation}
Using the decomposition
\[
r(g(w_L,w_H))=\bar r(w_L)+\delta(w),
\qquad
\delta(w):=r(g(w_L,w_H))-\bar r(w_L),
\]
we obtain
\begin{equation}
    q^\star(w)\propto q(w)\exp\big(\bar r(w_L)+\delta(w)\big),
\end{equation}
where $|\delta(w)|\le \epsilon$ by assumption.

The low-frequency-only tilted distribution is defined as
\begin{equation}
    \tilde q^\star(w)\propto q(w)\exp\big(\bar r(w_L)\big).
\end{equation}

Let $Z$ and $\tilde Z$ denote the corresponding normalization constants:
\[
q^\star(w)=\frac{1}{Z}q(w)\exp\big(\bar r(w_L)+\delta(w)\big),
\qquad
\tilde q^\star(w)=\frac{1}{\tilde Z}q(w)\exp\big(\bar r(w_L)\big).
\]
Therefore,
\begin{equation}
    \log \frac{q^\star(w)}{\tilde q^\star(w)}
    =
    \delta(w)+\log\frac{\tilde Z}{Z}.
\end{equation}

Next, using
\[
\tilde q^\star(w)=\frac{1}{\tilde Z}q(w)\exp(\bar r(w_L)),
\]
we can write
\begin{align}
    Z
    &= \int q(w)\exp(\bar r(w_L))\exp(\delta(w))\,dw \\
    &= \tilde Z \int \tilde q^\star(w)\exp(\delta(w))\,dw \\
    &= \tilde Z\,\mathbb E_{\tilde q}[\exp(\delta(w))].
\end{align}
Since $|\delta(w)|\le \epsilon$, we have
\[
e^{-\epsilon}\le e^{\delta(w)}\le e^\epsilon,
\]
and hence
\[
e^{-\epsilon}\le \frac{Z}{\tilde Z}\le e^\epsilon.
\]
Equivalently,
\begin{equation}
    \left|\log\frac{\tilde Z}{Z}\right|\le \epsilon.
\end{equation}

Moreover, since $|\delta(w)|\le \epsilon$, it follows that
\begin{equation}
    -\epsilon \le \mathbb E_{q^\star}[\delta(w)] \le \epsilon.
\end{equation}

Therefore,
\begin{align}
    D_{\mathrm{KL}}(q^\star\|\tilde q^\star)
    &= \mathbb E_{q^\star}\left[\log\frac{q^\star(w)}{\tilde q^\star(w)}\right] \\
    &= \mathbb E_{q^\star}[\delta(w)] + \log\frac{\tilde Z}{Z} \\
    &\le \epsilon + \epsilon \\
    &= 2\epsilon.
\end{align}
\end{proof}

\subsection{(Proof of Proposition \ref{proposition2}) Derivation of Tractable Loss Function} \label{proof_proposition2}

\paragraph{Derivation of KL Divergence Objective.}
We aim to solve the following optimization problem:
\begin{equation}
    D_{\mathrm{KL}}(q_L^\phi \,\|\, \tilde q^\star_L) 
    = \mathbb{E}_{w_L \sim q_L^\phi} 
    \left[ 
    \log \frac{q_L^\phi(w_L)}{\tilde{q}^\star_L(w_L)} 
    \right].
\end{equation}

Introducing the normalization constant $Z$, we can write the low-frequency-only tilted distribution as
\begin{equation}
    \tilde{q}^\star_L(w_L)
    =
    \frac{1}{Z}
    q_L(w_L)\exp(\bar{r}(w_L)).
\end{equation}

Taking the logarithm, we obtain
\begin{equation}
    \log \tilde{q}^\star_L(w_L)
    =
    \log q_L(w_L)
    + \bar{r}(w_L)
    - \log Z.
\end{equation}

Substituting this into the KL divergence, we obtain
\begin{equation}
    D_{\mathrm{KL}}(q_L^\phi \,\|\, \tilde{q}^\star_L)
    =
    \mathbb{E}_{w_L \sim q_L^\phi}
    \left[
    \log q_L^\phi(w_L)
    - \log q_L(w_L)
    - \bar{r}(w_L)
    + \log Z
    \right].
\end{equation}

Separating the terms, we have
\begin{equation}
    D_{\mathrm{KL}}(q_L^\phi \,\|\, \tilde{q}^\star_L)
    =
    \mathbb{E}_{w_L \sim q_L^\phi}
    \left[
    \log q_L^\phi(w_L)
    - \log q_L(w_L)
    \right]
    - \mathbb{E}_{w_L \sim q_L^\phi}[\bar{r}(w_L)]
    + \log Z.
\end{equation}

Thus, the objective can be written as
\begin{equation} \label{appedix:KL_obj}
    D_{\mathrm{KL}}(q_L^\phi \,\|\, \tilde{q}^\star_L)
    =
    D_{\mathrm{KL}}(q_L^\phi \,\|\, q_L)
    - \mathbb{E}_{w_L \sim q_L^\phi}[\bar{r}(w_L)]
    + C,
\end{equation}
where $C = \log Z$ is constant with respect to $\phi$, and can therefore be ignored during optimization.

\paragraph{Practical Approximation of the Low-Frequency Reward.}
In practice, the low-frequency reward $\bar{r}(w_L)$ is intractable to compute, as it requires taking an expectation over the high-frequency components $w_H$. However, under Assumption~\eqref{assumption}, $\bar{r}(w_L)$ can be well approximated by $r(g(w))$ for any $w$, as their discrepancy is uniformly bounded by $\epsilon$. Consequently, we replace $\bar{r}(w_L)$ with $r(g(w))$ in practice, leading to a tractable training objective.

Based on this approximation, we derive the practical training objective from the low-frequency KL formulation:
\begin{equation}
    \mathcal{L}(\phi)
    =
    D_{\mathrm{KL}}(q_L^\phi \,\|\, q_L)
    -
    \mathbb{E}_{w \sim q^\phi}[r(g(w))].
\end{equation}

The true objective in the low-frequency space is given by Eq. \eqref{appedix:KL_obj}. The proposed objective $\mathcal{L}(\phi)$ serves as a surrogate for this quantity, and their discrepancy is uniformly bounded by $\epsilon$. This leads to the following bound:
\begin{equation} \label{appendix:final_train_loss}
    \mathcal{L}(\phi) -\epsilon \lesssim  D_{\mathrm{KL}}(q_L^\phi \| q_L)
     - \mathbb{E}_{w_L \sim q_L^\phi}[\bar{r}(w_L)] \lesssim
     \mathcal{L}(\phi) +\epsilon.
\end{equation}

\paragraph{Exact KL divergence.}
The proof is based on the approach of \cite{eyring2025noise}. We define the transformation induced by the network in the low-frequency space as
\begin{equation}
    T_\phi(w_L) = w_L + h_\phi(w_L),
\end{equation}
where $h_\phi$ is a residual predictor. Let $q_L = \mathcal{N}(0, I)$ denote the standard Gaussian prior over $w_L$, and let $q_L^\phi = (T_\phi)_\# q_L$ denote the pushforward distribution induced by $T_\phi$.

To apply the change-of-variables formula, the transformation $T_\phi$ must be invertible with a continuously differentiable inverse. The following lemma provides a sufficient condition for this property.

\begin{lemma}[Invertibility of the residual map]
\label{lemma:invertibility_vector}
Let $h_\phi : \mathbb{R}^{Nk} \to \mathbb{R}^{Nk}$ be a continuously differentiable function. If there exists a constant $M$ such that
\begin{equation} \label{appendix_jacobian_condition}
        \|J_{h_\phi}(w_L)\|_2 \le M <1
    \quad \text{for all } w_L \in \mathbb{R}^{Nk},
\end{equation}
then $T_\phi$ is a bijection from $\mathbb{R}^{Nk}$ onto $\mathbb{R}^{Nk}$, and its inverse is continuously differentiable.
\end{lemma}

\begin{proof}
Since $h_\phi$ is continuously differentiable and satisfies Eq.~\eqref{appendix_jacobian_condition}, it follows from the mean value theorem that $h_\phi$ is $M$-Lipschitz:
\begin{equation}
    \|h_\phi(x) - h_\phi(y)\|
    \le M \|x-y\|
    \qquad \forall x,y \in \mathbb{R}^{Nk}.
\end{equation}

\textit{(i) Injectivity. }
For any $x,y \in \mathbb{R}^{Nk}$,
\begin{equation}
\begin{aligned}
    \|T_\phi(x) - T_\phi(y)\|
    &= \|(x-y) + (h_\phi(x)-h_\phi(y))\| \\
    &\ge \|x-y\| - \|h_\phi(x)-h_\phi(y)\| \\
    &\ge (1-M)\|x-y\|.
\end{aligned}
\end{equation}
Since $1-M > 0$, it follows that $T_\phi(x)=T_\phi(y)$ implies $x=y$. Hence, $T_\phi$ is injective.

\textit{(ii) Surjectivity. }
Fix any $z \in \mathbb{R}^{Nk}$ and define
\begin{equation}
    \Psi_z(x) := z - h_\phi(x).
\end{equation}
Then for any $x,y$,
\begin{equation}
    \|\Psi_z(x)-\Psi_z(y)\|
    = \|h_\phi(x)-h_\phi(y)\|
    \le M\|x-y\|.
\end{equation}
Thus, $\Psi_z$ is a contraction with constant $L < 1$. By the Banach fixed-point theorem, there exists a unique $x^\star$ such that
\begin{equation}
    x^\star = \Psi_z(x^\star) = z - h_\phi(x^\star).
\end{equation}
Rearranging gives
\begin{equation}
    z = x^\star + h_\phi(x^\star) = T_\phi(x^\star).
\end{equation}
Since $z$ is arbitrary, $T_\phi$ is surjective.

\textit{(iii) Differentiability of the inverse. }
The Jacobian of $T_\phi$ is
\begin{equation}
    J_{T_\phi}(w_L) = I + J_{h_\phi}(w_L).
\end{equation}
we have Eq. \eqref{appendix_jacobian_condition}, hence, the spectral radius satisfies
\begin{equation}
    \rho(J_{h_\phi}(w_L)) \le \|J_{h_\phi}(w_L)\|_2 < 1.
\end{equation}

Therefore, for every eigenvalue $\lambda$ of $J_{h_\phi}(w_L)$, we have $|\lambda|<1$, so
\[
1+\lambda \neq 0.
\]
It follows that $I+J_{h_\phi}(w_L)$ is invertible for all $w_L$, i.e.,
\[
\det(J_{T_\phi}(w_L)) \neq 0 \qquad \forall w_L.
\]
By the inverse function theorem, $T_\phi$ admits a locally defined inverse around every point, and this inverse is continuously differentiable. Since Parts (i) and (ii) establish that $T_\phi$ is globally bijective, these local inverses are consistent and define a global inverse $T_\phi^{-1}$, which is also continuously differentiable.
\end{proof}

Under the condition of Lemma~\ref{lemma:invertibility_vector} (i.e., $\|J_{h_\phi}(w_L)\|_2 \le M <1$), we can use change of variables formula. The density of $q_L^\phi$ is given by
\begin{equation}
    q_L^\phi(\hat{w}_L)
    =
    q_L(w_L)\,
    \left|
    \det\!\left(
    \frac{\partial T_\phi(w_L)}{\partial w_L}
    \right)
    \right|^{-1},
\end{equation}
where $\hat{w}_L = T_\phi(w_L)$. Since
\[
\frac{\partial T_\phi(w_L)}{\partial w_L}
=
I + J_{h_\phi}(w_L),
\]
this yields
\begin{equation}
    q_L^\phi(\hat{w}_L)
    =
    q_L(w_L)\,
    \left|
    \det\!\bigl(I + J_{h_\phi}(w_L)\bigr)
    \right|^{-1}.
\end{equation}

Using the definition of KL divergence,
\begin{equation}
    D_{\mathrm{KL}}(q_L^\phi \,\|\, q_L)
    =
    \mathbb{E}_{\hat{w}_L \sim q_L^\phi}
    \left[
    \log \frac{q_L^\phi(\hat{w}_L)}{q_L(\hat{w}_L)}
    \right],
\end{equation}
we substitute the expression of $q_L^\phi(\hat{w}_L)$ to obtain
\begin{equation}
    \log \frac{q_L^\phi(\hat{w}_L)}{q_L(\hat{w}_L)}
    =
    \log q_L(w_L) - \log q_L(\hat{w}_L)
    - \log \left| \det\!\bigl(I + J_{h_\phi}(w_L)\bigr) \right|.
\end{equation}

Since $q_L$ is a standard Gaussian distribution, its log-density is quadratic. 
Thus,
\begin{equation}
    \log q_L(w_L) - \log q_L(\hat{w}_L)
    =
    \frac{1}{2}\|\hat{w}_L\|^2 - \frac{1}{2}\|w_L\|^2.
\end{equation}

Substituting $\hat{w}_L = w_L + h_\phi(w_L)$ and expanding,
\begin{equation}
\begin{aligned}
    \frac{1}{2}\|\hat{w}_L\|^2 - \frac{1}{2}\|w_L\|^2
    &=
    \frac{1}{2}\|w_L + h_\phi(w_L)\|^2 - \frac{1}{2}\|w_L\|^2 \\
    &=
    \frac{1}{2}\|h_\phi(w_L)\|^2 
    + \langle w_L, h_\phi(w_L)\rangle.
\end{aligned}
\end{equation}

The first term corresponds to the magnitude of the residual update, 
while the second term captures the interaction between the input $w_L$ and the update direction.

To simplify the second term, we apply Stein's lemma under the Gaussian base measure $w_L \sim \mathcal N(0,I)$. Assume that $h_\phi$ is continuously differentiable and satisfies
\begin{equation}
    \mathbb{E}\big[\|h_\phi(w_L)\|\big] < \infty,
    \qquad
    \mathbb{E}\big[\|J_{h_\phi}(w_L)\|_2\big] < \infty,
\end{equation}
so that integration by parts applies and the boundary terms vanish under the Gaussian measure. In our setting, these conditions are ensured by the Jacobian bound in Eq.~\eqref{appendix_jacobian_condition}, which implies that $h_\phi$ is Lipschitz and hence has at most linear growth.

Applying Stein's lemma yields
\begin{equation}
    \mathbb{E}\bigl[\langle w_L, h_\phi(w_L)\rangle\bigr]
    =
    \mathbb{E}\bigl[\mathrm{Tr}(J_{h_\phi}(w_L))\bigr].
\end{equation}

Combining the above results, we obtain the exact expression
\begin{equation}
D_{\mathrm{KL}}(q_L^\phi \,\|\, q_L)
=
\mathbb{E}_{w_L \sim q_L}
\left[
\frac{1}{2}\|h_\phi(w_L)\|^2
+
\mathrm{Tr}\bigl(J_{h_\phi}(w_L)\bigr)
-
\log \det\!\bigl(I + J_{h_\phi}(w_L)\bigr)
\right].
\end{equation}

This expression shows that the KL divergence decomposes into a quadratic penalty on $h_\phi$ and a Jacobian-based correction capturing local geometric distortion.

\paragraph{Approximation of the Jacobian term.}

The exact KL expression derived above contains a Jacobian-dependent correction term,
\begin{equation}
    \mathrm{Tr}(J_{h_\phi}(w_L)) - \log \det\!\bigl(I + J_{h_\phi}(w_L)\bigr).
\end{equation}
While this term is necessary for exact density transformation, it is computationally expensive, as it requires evaluating both the Jacobian trace and the log-determinant.

To obtain a tractable objective, we analyze this term under the small-Jacobian assumption in Eq.~\eqref{appendix_jacobian_condition} as in Lemma~\ref{lemma:invertibility_vector}. Let
\begin{equation}
    A := J_{h_\phi}(w_L).
\end{equation}

Since $\|A\|_2 \le M < 1$, its spectral radius satisfies $\rho(A) < 1$, and hence the matrix logarithm admits the convergent expansion
\begin{equation}
    \log(I+A) = \sum_{j=1}^{\infty} \frac{(-1)^{j+1}}{j} A^j.
\end{equation}
Using the identity $\log\det(I+A)=\mathrm{Tr}(\log(I+A))$, we obtain
\begin{equation}
    \log\det(I+A)
    =
    \sum_{j=1}^{\infty} \frac{(-1)^{j+1}}{j}\mathrm{Tr}(A^j).
\end{equation}
Therefore,
\begin{equation}
\mathrm{Tr}(A)-\log\det(I+A)
=
\sum_{j=2}^{\infty}\frac{(-1)^j}{j}\mathrm{Tr}(A^j).
\end{equation}

Taking absolute values and using $|\mathrm{Tr}(A^j)| \le Nk\,\|A\|_2^j \le Nk\,M^j$, we get
\begin{equation}
\left|
\mathrm{Tr}(A)-\log\det(I+A)
\right|
\le
Nk\sum_{j=2}^{\infty}\frac{M^j}{j}
=
Nk\bigl(-\log(1-M)-M\bigr).
\end{equation}

Substituting this bound into the KL expression yields
\begin{equation}
\left|
D_{\mathrm{KL}}(q_L^\phi \,\|\, q_L)
-
\frac{1}{2}\mathbb{E}_{w_L \sim q_L}\|h_\phi(w_L)\|^2
\right|
\le Nk\bigl(-\log(1-M)-M\bigr).
\end{equation}

When $M \ll 1$, the bound admits the second-order approximation
\begin{equation}
    -\log(1-M)-M = \frac{M^2}{2} + O(M^3),
\end{equation}
which becomes negligible as $M \to 0$. Consequently,
\begin{equation}
    D_{\mathrm{KL}}(q_L^\phi \,\|\, q_L)
    \approx
    \frac{1}{2}\mathbb{E}_{w_L \sim q_L}\|h_\phi(w_L)\|^2.
\end{equation}

This shows that, in the small-Jacobian regime, the KL divergence is well approximated by a simple quadratic penalty on the residual update $h_\phi$.

\paragraph{Final practical objective.}

Applying this approximation yields the practical training loss
\begin{equation}
    \mathcal{L}(\phi)
    =
    \mathbb{E}_{w \sim q}
    \left[
    \frac{1}{2}\|h_\phi(w_L)\|^2
    -
    r\big(g(w_L + h_\phi(w_L), w_H)\big)
    \right].
\end{equation}

This objective admits a clear interpretation: the first term acts as a regularizer that penalizes large deviations from the base Gaussian prior in the low-frequency subspace, while the second term encourages updates that improve the reward.

\section{Additional Details on Model Architecture and Loss Function}

\subsection{Model Architecture of the LENS Framework} \label{architecture_explain}

\paragraph{Patch PCA Codec.}
The Patch PCA Codec transforms latent noise into a patch-wise PCA coefficient space and reconstructs it back to the latent space.

Let $z \in \mathbb{R}^{C \times H \times W}$ denote a latent noise representation, where $C$, $H$, and $W$ denote the number of channels, height, and width, respectively. We extract non-overlapping patches using an \texttt{unfold} operation with patch size $s$, producing a sequence of patch vectors:
\begin{equation}
    \mathbf{s} = \{ s_i \}_{i=1}^N, \quad s_i \in \mathbb{R}^{d},
\end{equation}
where $N = \frac{H}{s} \cdot \frac{W}{s}$ and $d = C s^2$.

Each patch is centered using a precomputed mean $\mu \in \mathbb{R}^{d}$:
\begin{equation}
    \tilde{s}_i = s_i - \mu.
\end{equation}

Since $\mu$ is fixed, centering only shifts the mean of the PCA coefficients while preserving their covariance. For the theoretical analysis in Section \ref{theoretical analysis}, we absorb this deterministic mean shift into the coefficient coordinates and use re-centered variables with a standard Gaussian prior.

We project the centered patches onto a truncated PCA basis $V' \in \mathbb{R}^{d \times k}$, whose columns are orthonormal (i.e., ${V'}^\top V' = I$):
\begin{equation}
    w_{i,L} = \tilde{s}_i V',
\end{equation}
where $w_{i,L} \in \mathbb{R}^{k}$ are the retained low-frequency PCA coefficients and $k \ll d$.

The high-frequency component is defined as:
\begin{equation}
    y_i = \tilde{s}_i - w_{i,L} {V'}^\top,
\end{equation}
where $y_i$ lies in the orthogonal complement of the subspace spanned by $V'$.

Collecting all patches, we obtain:
\begin{equation}
    w_L = \{ w_{i,L} \}_{i=1}^N \in \mathbb{R}^{N \times k}, \quad 
    y = \{ y_i \}_{i=1}^N \in \mathbb{R}^{N \times d}.
\end{equation}

Given modified coefficients $\hat{w}_{i,L}$, reconstruction is performed as:
\begin{equation}
    \hat{s}_i = \mu + \hat{w}_{i,L} {V'}^\top + y_i.
\end{equation}

The reconstructed patches are merged back into the latent space using a \texttt{fold} operation.

%\textbf{Additional Details on the Architecture}
This formulation is mathematically equivalent to performing a full PCA projection using an orthonormal basis $V \in \mathbb{R}^{d \times d}$, obtaining full coefficients $w_i = \tilde{s}_i V$, modifying only the first $k$ coefficients, and reconstructing via $w_i V^\top$. 

Specifically, let $V = [V', V'']$, where $V'$ contains the first $k$ principal components and $V''$ contains the remaining components. Then, since $V$ is orthonormal, we can decompose:
\begin{equation}
    \tilde{s}_i = w_{i,L} {V'}^\top + w_{i,H} {V''}^\top,
\end{equation}
where $w_{i,H}$ denotes the remaining high-frequency coefficients.

In our formulation, the orthogonal component $y_i$ corresponds exactly to the contribution of the high-frequency components, i.e., the component of $\tilde{s}_i$ that lies in the subspace orthogonal to the span of the top-$k$ principal components $V'$:
\begin{equation}
    y_i = w_{i,H} {V''}^\top.
\end{equation}

Therefore, modifying only $w_{i,L}$ while keeping $y_i$ fixed is equivalent to modulating only the first $k$ PCA coefficients in the full PCA space.

\paragraph{LENS Network.}
The LENS Network is a transformer-based model~\cite{vaswani2017attention} that predicts residual updates to the PCA coefficients. 

The input to the network is a sequence of low-frequency PCA coefficients:
\begin{equation}
    w_L \in \mathbb{R}^{N \times k},
\end{equation}
where $N$ is the number of patches and $k$ is the number of low-frequency coefficients used for modulation. Each token corresponds to a patch and contains its $k$-dimensional PCA coefficient vector. Each coefficient vector is first projected into a hidden dimension $h$:
\begin{equation}
    \mathbf{h}^{(0)} = \mathrm{Linear}(w_L) + \mathbf{P}, \quad \mathbf{h}^{(0)} \in \mathbb{R}^{N \times h},
\end{equation}
where $\mathbf{P}$ is a sinusoidal positional embedding.

Given a text prompt embedding $c$, two conditioning paths are used: (1) a pooled text condition that is projected and added to all patch tokens (bias injection), and (2) token-wise text embeddings used in cross-attention.

Each transformer block applies self-attention, text cross-attention, and a feed-forward network (FFN), with residual connections and layer normalization.

After passing through $L$ transformer blocks, the final hidden representation is projected back to the coefficient space:
\begin{equation}
    \Delta w_L = \mathrm{Linear}(\mathrm{LN}(\mathbf{h}^{(L)})),
\end{equation}
where $\Delta w_L \in \mathbb{R}^{N \times k}$ represents the predicted coefficient updates.

\subsection{Composition of the Reward-based Loss Function} \label{reward_explain}

We employ a composite reward function to guide the optimization of the coefficient modulation network. The overall reward is defined as a weighted combination of multiple complementary scoring functions that capture semantic alignment, human preference, and perceptual quality.

Given a generated image $x$ conditioned on a text prompt $c$, the total reward is defined as
\begin{equation}
r(x, c) = \sum_{i} \lambda_i \, r_i(x, c),
\end{equation}
where $\{r_i\}$ denotes individual reward components and $\{\lambda_i\}$ are weighting coefficients. The proposed reward combines complementary signals: CLIP~\cite{hessel2021clipscore} provides a strong signal for semantic alignment, HPSv2.1 (Human Preference Score)~\cite{wu2023human} captures absolute perceptual quality, ImageReward~\cite{xu2023imagereward} promotes visual fidelity by penalizing generative artifacts, and PickScore~\cite{kirstain2023pick} models relative human preference. This combination yields a more stable and well-balanced optimization signal compared to relying on a single reward model.

In practice, we combine four reward functions:
\begin{equation}
r(x, c) =
\lambda_{\text{CLIP}} r_{\text{CLIP}}(x, c)
+
\lambda_{\text{HPS}} r_{\text{HPS}}(x, c)
+
\lambda_{\text{IR}} r_{\text{IR}}(x, c)
+
\lambda_{\text{Pick}} r_{\text{Pick}}(x, c).
\end{equation}
We set the weights as $\lambda_{\text{CLIP}} = 0.01$, $\lambda_{\text{HPS}} = 5.0$, $\lambda_{\text{IR}} = 1.0$, and $\lambda_{\text{Pick}} = 0.05$.

\paragraph{CLIP Score.}
The CLIP reward \cite{hessel2021clipscore} measures semantic alignment by computing the cosine similarity between image and text embeddings in a shared latent space. It encourages the generated image to remain consistent with the input prompt.

\paragraph{HPSv2.1.}
HPSv2.1 \cite{wu2023human} is a model fine-tuned on human preference data to predict aesthetic quality. It captures perceptual attributes such as composition, lighting, and detail.

\paragraph{ImageReward.}
ImageReward \cite{xu2023imagereward} is trained on large-scale human feedback to evaluate overall image quality conditioned on text prompts. It provides a scalar reward reflecting both visual fidelity and semantic alignment.

\paragraph{PickScore.}
PickScore \cite{kirstain2023pick} is trained on pairwise human preference data using a Bradley--Terry objective. It predicts relative preference between images for the same prompt, capturing fine-grained ranking signals.

\newpage

\section{Experimental Details} \label{appendix_implementation_detail}
\subsection{Implementation Details}
\paragraph{PCA Basis Extraction.}
Before training each model, we precompute the PCA basis for latent patches. To this end, we use ImageNetV2 (MIT License)~\cite{recht2019imagenet}, a dataset designed to evaluate generalization performance. ImageNetV2 preserves the same 1,000-class structure as ImageNet \cite{russakovsky2015imagenet} while consisting of 10 newly collected images per class, resulting in a total of 10,000 images without overlap with the original training data. By covering a wide range of classes and visual variations, it provides a reasonable approximation of natural image statistics while mitigating overfitting to a specific dataset distribution.

The PCA basis is computed as follows. First, images are encoded into latent representations using the VAE of the corresponding model. The latent features are then divided into patches according to a predefined patch size and vectorized. Next, the mean vector of the patches is computed, and each patch is centered by subtracting the mean. Based on these centered vectors, a covariance matrix is constructed, from which eigenvectors are obtained. These eigenvectors are used as the PCA basis during both training and inference for each model.

\paragraph{Models and Baselines.}
In this work, we employ three distilled diffusion models: SD-Turbo\footnote{\url{https://huggingface.co/stabilityai/sd-turbo}} \cite{sauer2024adversarial}, Hyper-SDXL\footnote{\url{https://huggingface.co/ByteDance/Hyper-SD}} \cite{ren2024hyper}, and SANA-Sprint\footnote{\url{https://huggingface.co/Efficient-Large-Model/Sana_Sprint_0.6B_1024px_diffusers}} \cite{chen2025sana}. For Hyper-SDXL, we use the 1-step SDXL UNet checkpoint provided in the official repository. 

As baselines, we employ ReNo (MIT License)\footnote{\url{https://github.com/ExplainableML/ReNO}}
 \cite{eyring2024reno} and HyperNoise (MIT License)\footnote{\url{https://github.com/ExplainableML/HyperNoise}}
 \cite{eyring2025noise}.

ReNo is a test-time optimization method that refines latent representations during inference to better align generated samples with the target distribution. Specifically, it performs iterative optimization at test time without modifying the model parameters. In our experiments, we set the number of optimization iterations to 20 and follow the official implementation for all other hyperparameters. For SANA-Sprint, we reimplement ReNo based on the original codebase to ensure compatibility with this model.

HyperNoise is a hypernetwork-based approach that learns to generate input-conditioned noise modulations to steer the denoising process toward better quality samples. We follow the official implementation for all hyperparameters, including setting the LoRA rank to 128. Since the original implementation does not support Hyper-SDXL, we reimplement HyperNoise for this model based on the provided codebase.

\paragraph{Settings.}
In LENS, the base generator is pretrained and kept frozen. In the trainable LENS network, most layers use the default PyTorch initialization, while the final output projection is zero-initialized, and the cross-attention gate logits are initialized to -2.0 (gain $\approx 0.119$), resulting in near-zero initial perturbations under the default modulation setting. This design is intended to stabilize training by starting from minimal perturbations, with updates initially focusing on the output layer and gradually propagating to the entire network, thereby reducing variance and improving convergence in reward-based learning.

Across all experiments, the guidance scale is set to 0 during both training and evaluation. For all three models—SD-Turbo, Hyper-SDXL, and SANA-Sprint—the LENS network is trained at a resolution of 512×512. This choice is primarily motivated by computational efficiency and resource constraints, while maintaining adequate text–image alignment. During evaluation, images are generated at resolutions of 512, 1024, and 1024, respectively. The number of generation steps is fixed at 1 for both training and evaluation, except for the ablation experiments in the Appendix where the number of steps is varied.

When evaluating on GenEval2~\cite{kamath2025geneval}, we adopt Soft-TIFA AM as the scoring metric and generate images using different seeds for each prompt to mitigate dependence on random seed selection. To further reduce statistical variance, we repeat the experiments with three different base seeds and report the averaged results, where each base seed defines a distinct set of per-prompt seeds.

All training and experiments were conducted using Python 3.11.0 on NVIDIA RTX 3090 GPUs, with a 64-core Intel Xeon Gold 6226R CPU. The training data type was set to bfloat16 (bf16) by default. To increase the effective batch size without significantly increasing VRAM usage, gradient accumulation was employed. Additionally, gradient checkpointing was applied to both the text encoder and parts of the reward model to improve memory efficiency. When using two RTX 3090 GPUs, training each LENS network typically required approximately 4 to 6 days, although the exact duration varied slightly depending on the model architecture and overall GPU utilization.

\subsection{Hyperparameters}
The hyperparameters used to train our LENS Network for each model are summarized in Table \ref{appendix_hyperparameter}.

\begin{table}[!ht]
\caption{Hyperparameters for training the LENS network}
\vspace{3pt}
\centering
\resizebox{0.92\textwidth}{!}{
%\scriptsize
\renewcommand{\arraystretch}{1.2}
\begin{tabular}{lcccc}
\toprule
 & \multicolumn{3}{c}{\textbf{Low-frequency Eigen Noise Shaping (LENS)}} \\
\cmidrule(lr){2-4} 
Model 
& SD-Turbo 
& Hyper-SDXL
& SANA-Sprint    \\
\midrule
Patch size & 4 & 4 & 4\\
Number of low-frequency coefficients ($k$) & 32 & 32 & 256\\
Number of layers (blocks) & 4 & 4 & 8\\
Number of attention heads & 4 & 4 & 8\\
Gradient norm clipping & 1.0 & 1.0 & 1.0 \\
Learning rate & $2e{-5}$ & $5e{-6}$ & $1e{-3}$\\
Optimizer & AdamW & AdamW & AdamW\\
Batch size & 2 & 2 & 4 \\
Accumulation steps & 15 & 8 & 6 \\
Training epochs & $\approx 15$ & $\approx 20$ & $\approx 20$  \\
Number of training prompts & $\approx 70k$ & $\approx 70k$ & $\approx 70k$  \\
\bottomrule
\end{tabular}
}

\label{appendix_hyperparameter}
\end{table}

\section{Additional Experimental Results} \label{additional_exp}

\subsection{Spatial-Frequency Characteristics of PCA Components} \label{low_frequency_exp}
To verify the use of the term “low-frequency” for the leading PCA components, we analyze the spatial-frequency characteristics of the PCA bases used by LENS. For each PCA eigenvector, we reshape it into its latent-patch representation, compute its 2D Fourier transform, and measure the spectral centroid of its power spectrum. We then compare the average spectral centroid of the leading $k$ components used by LENS with that of the remaining $d-k$ components, using $k=32$ for SD-Turbo and Hyper-SDXL and $k=256$ for SANA-Sprint, following the PCA component settings used in our experiments. We additionally report the Spearman correlation between PCA rank and spectral centroid. Table~\ref{tab:low_frequency_analysis} summarizes the results.

\begin{table}[!ht]
\centering
\renewcommand{\arraystretch}{1.25}
\caption{Spatial-frequency characteristics of the PCA components used by LENS}
\vspace{3pt}
\label{tab:low_frequency_analysis}
\begin{tabular}{lccc}
\toprule
Backbone 
& \makecell{Leading $k$ spectral\\centroid mean}
& \makecell{Remaining $d-k$ spectral\\centroid mean}
& Spearman correlation \\
\midrule
SD-Turbo    & 0.3456 & 0.4491 & 0.4293 \\
Hyper-SDXL  & 0.3670 & 0.4277 & 0.2960 \\
SANA-Sprint & 0.3260 & 0.4689 & 0.8751 \\
\bottomrule
\end{tabular}
\vspace{10pt}
\end{table}

Across all three backbones, the leading $k$ components exhibit lower average spectral centroids than their complementary components. Moreover, the positive Spearman correlations indicate an overall tendency for spatial frequency to increase with PCA rank. These results support our use of “low-frequency” as a shorthand for the empirically lower-frequency characteristics of the leading PCA subspace, rather than as an exact identification with the Fourier basis.

\subsection{Empirical Validation of Assumption~\ref{assumption}} \label{assumption_1_validation}
To examine the empirical plausibility of Assumption~\ref{assumption}, we evaluate the sensitivity of the reward to perturbations in the low-frequency PCA coefficients $w_L$ and the complementary coefficients $w_H$. Using SD-Turbo, we construct 250 prompt--noise pairs from 25 prompts and 10 anchor noise samples. For each pair, we perform 30 independent resamplings under two conditions: (i) varying only $w_L$ while fixing $w_H$, and (ii) varying only $w_H$ while fixing $w_L$. The reward is computed using the same weighted combination of CLIP, HPS, ImageReward, and PickScore used during training. We quantify reward sensitivity using the mean absolute deviation and the standard deviation across repeated resamplings. We additionally report the $H/L$ ratio, defined as the reward variation induced by perturbing $w_H$ divided by that induced by perturbing $w_L$. Thus, an $H/L$ ratio below 1 indicates that the complementary coefficients have a smaller effect on the reward than the low-frequency coefficients. The results are summarized in Table~\ref{tab:assumption_validation}.

\begin{table}[!ht]
\centering
\renewcommand{\arraystretch}{1.3}
\caption{Reward sensitivity to perturbations of coefficients}
\vspace{3pt}
\label{tab:assumption_validation}
\begin{tabular}{lccc}
\toprule
 & $w_H$ varied & $w_L$ varied & $H/L$ ratio $\downarrow$ \\
\midrule
Mean absolute deviation & 0.2152 & 0.4925 & 0.4371 \\
Mean standard deviation & 0.2729 & 0.6211 & 0.4393 \\
\bottomrule
\end{tabular}
\end{table}

The two sensitivity measures yield highly consistent results. In both cases, the $H/L$ ratio is approximately 0.44, substantially below 1, indicating that perturbing the complementary coefficients produces less than half the reward variation caused by perturbing the low-frequency coefficients. These results provide empirical evidence supporting the plausibility of Assumption~\ref{assumption} and are consistent with the motivation for restricting modulation to the low-frequency PCA subspace.

\subsection{Effect of the Number of Low-frequency Coefficients} \label{additional_exp_k}

We conduct an ablation study on the number of low-frequency coefficients $k$ using the SD-Turbo model. The patch size is fixed at $s=4$, and only $k$ is varied during training. For SD-Turbo, the latent channel dimension is 4, resulting in a full dimension of $d = 64$. Images are generated using the LENS framework and evaluated on the GenEval2 benchmark. Table~\ref{appendix_table_k} summarizes the results for each metric, and Figure~\ref{appendix_figure_k} compares the GenEval2 mean scores.

As shown, the best performance is achieved at $k = 32$. Increasing $k$ toward the full dimension $d$ degrades performance, as it involves modifying an unnecessarily large number of components. This suggests that effective noise modulation benefits from focusing on a restricted subspace, where limiting the degrees of freedom leads to more stable and efficient optimization.

\begin{table}[!ht]
\caption{Effect of the number of PCA coefficients $k$ on GenEval2 results}
\vspace{3pt}
\centering

\renewcommand{\arraystretch}{1.4} 
\setlength{\tabcolsep}{10pt}   

\begin{tabular}{c|ccccc}
\hline
\textbf{$k$} & \textbf{Object $\uparrow$} & \textbf{Attribute $\uparrow$} & \textbf{Count $\uparrow$} & \textbf{Position $\uparrow$} & \textbf{Verb $\uparrow$} \\
\hline
8  & 73.26 & 42.06 & 24.59 & 26.36 & 3.24 \\
16 & 73.50 & 40.83 & 23.91 & 28.17 & 1.53 \\
32 & \textbf{75.07} & \textbf{43.08} & \textbf{26.08} & \textbf{28.46} & \textbf{5.78} \\
64 & 74.02 & 42.55 & 25.05 & 28.21 & 5.33 \\
\hline
\end{tabular}

\label{appendix_table_k}
\end{table}

\begin{figure} [!ht]
    \centerline{\includegraphics[width=0.7\linewidth]{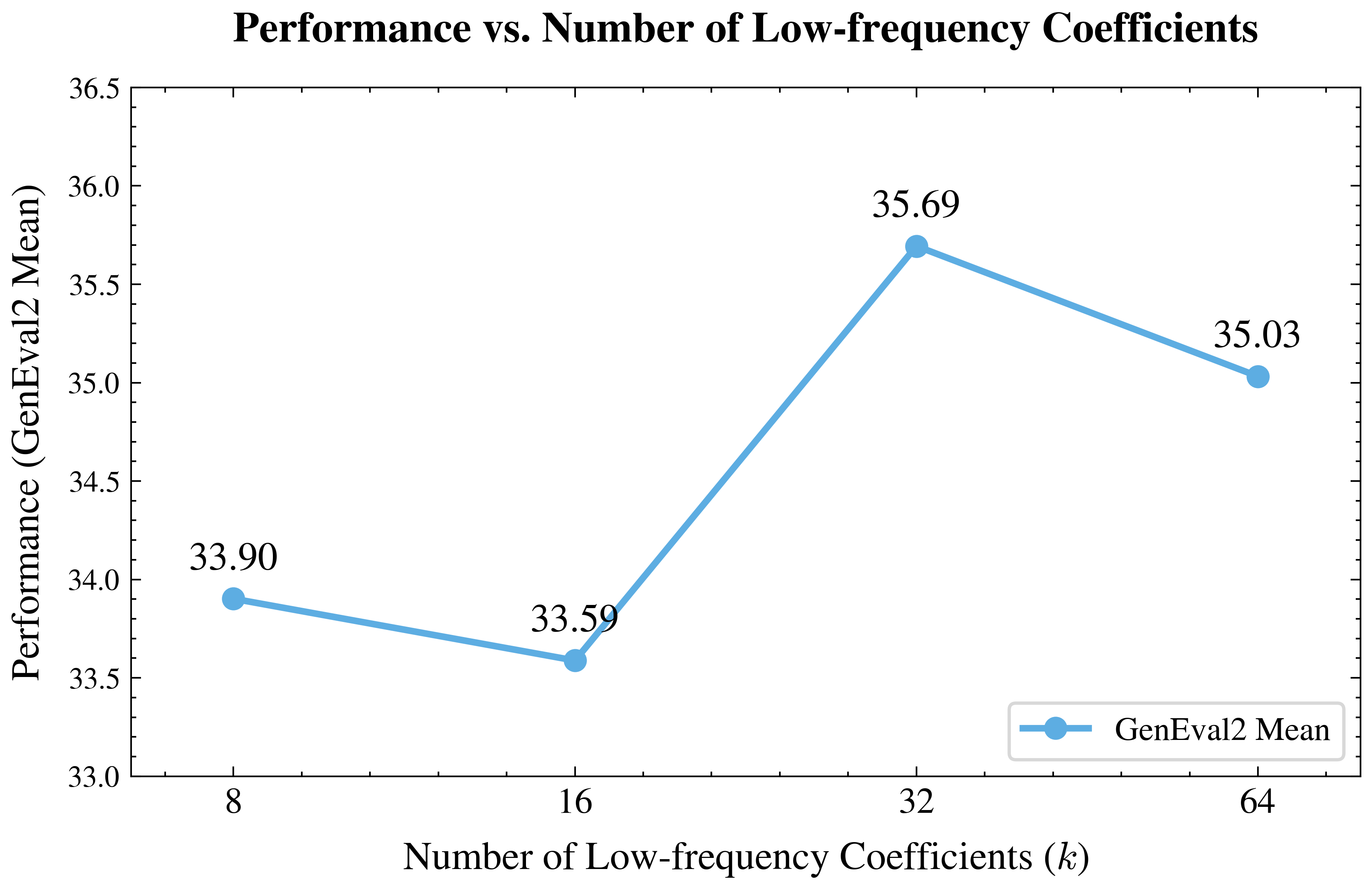}}
    \caption{Effect of the number of low-frequency coefficients $k$ on GenEval2 mean}
    \label{appendix_figure_k}
\end{figure}

\subsection{Effect of the Patch Size} \label{additional_exp_s}
We conduct an ablation study on the patch size $s$ using the SD-Turbo model. The number of low-frequency coefficients $k$ is fixed, and only $s$ is varied.  Images are generated using the LENS framework and evaluated on the GenEval2 benchmark. Table~\ref{appendix_table_s} summarizes the results for each metric, and Figure~\ref{appendix_figure_s} compares the GenEval2 mean scores.

The results show that the best performance is achieved at $s = 4$. This suggests that a moderate patch size strikes a good balance between preserving local details and capturing global structure, leading to improved performance.

\begin{table}[!ht]
\vspace{-5pt}
\caption{Effect of the patch size $s$ on GenEval2 results}
\vspace{3pt}
\centering
\renewcommand{\arraystretch}{1.4} 
\setlength{\tabcolsep}{10pt}   
\begin{tabular}{c|ccccc}
\hline
\textbf{$s$} & \textbf{Object $\uparrow$} & \textbf{Attribute $\uparrow$} & \textbf{Count $\uparrow$} & \textbf{Position $\uparrow$} & \textbf{Verb $\uparrow$} \\
\hline
2  & 73.71 & 42.37 & 23.16 & 28.24 & 4.35 \\
4 & \textbf{75.07} & \textbf{43.08} & \textbf{26.08} & \textbf{28.46} & \textbf{5.78} \\
8 & 73.59 & 41.06 & 23.44 & 28.27 & 3.67 \\
16 & 72.65 & 40.41 & 22.30 & 28.05 & 3.68 \\
\hline
\end{tabular}
\label{appendix_table_s}
\end{table}

\vspace{-7pt}

\begin{figure} [!ht]
    \centerline{\includegraphics[width=0.7\linewidth]{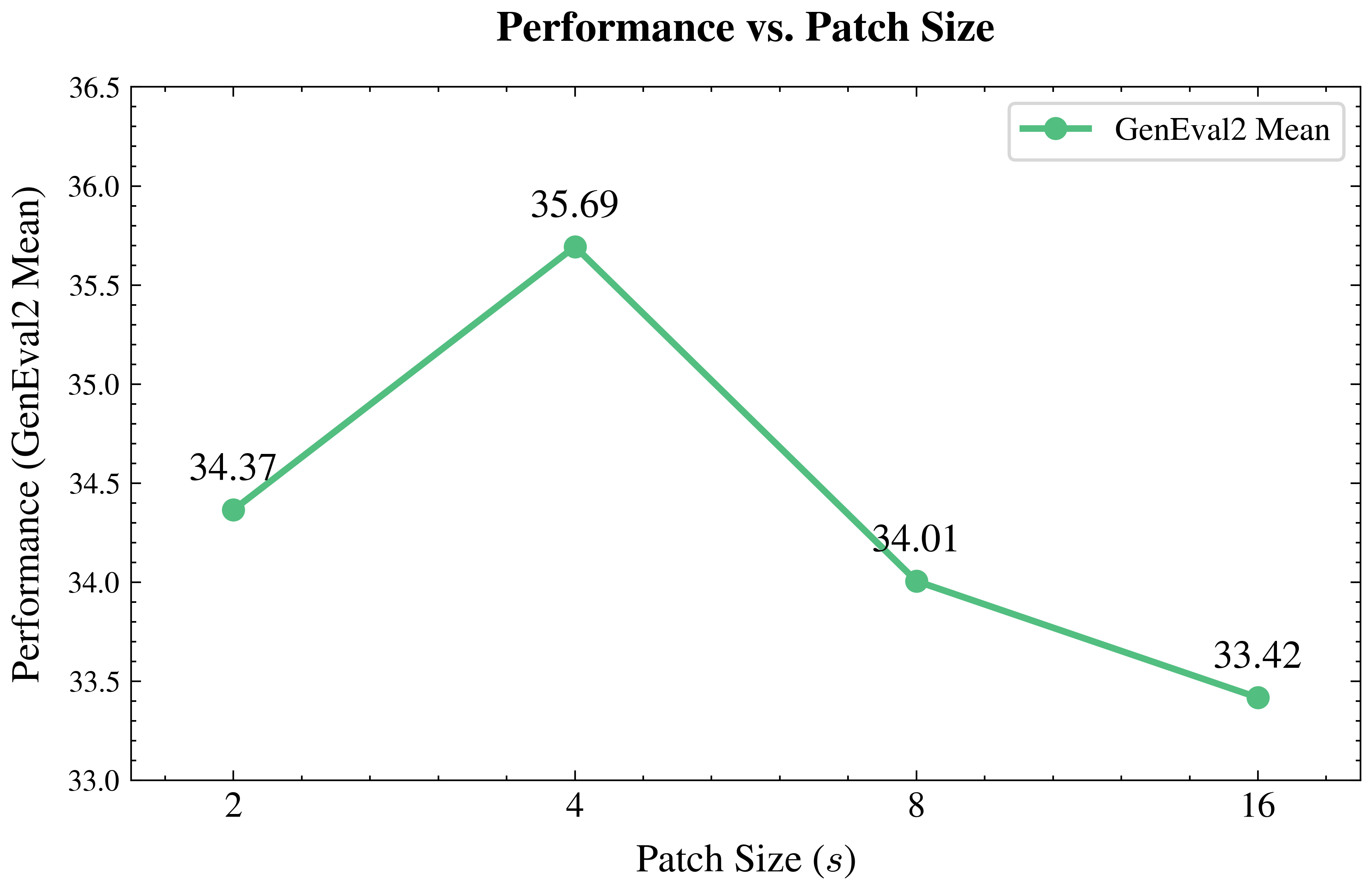}}
    \caption{Effect of the patch size $s$  on GenEval2 mean}
    \label{appendix_figure_s}
\end{figure}

\subsection{Effect of the Number of Transformer Layers} \label{additional_exp_layer}

\begin{wraptable}{r}{0.5\textwidth}
\vspace{-12pt}
\caption{Effect of the number of layers}
\vspace{-2pt}
\centering
\scriptsize
\resizebox{\linewidth}{!}{
\begin{tabular}{lcc}
\toprule 
\textbf{SD-Turbo + LENS}  & \textbf{GenEval2 Mean $\uparrow$} \\
\midrule
4 Layers & \textbf{35.70}  \\
8 Layers & 34.80   \\
12 Layers & 33.10   \\
\bottomrule
\end{tabular}
}
\vspace{-7pt}
\label{appendix_layer}
\end{wraptable}

Table~\ref{appendix_layer} presents the performance of LENS with varying numbers of transformer layers for the SD-Turbo model. As the number of layers increases, the GenEval2 mean score decreases from 35.70 to 33.10. This indicates that deeper architectures do not necessarily yield better performance. Since LENS operates on low-dimensional representations, increasing depth may introduce redundant transformations or hinder effective feature utilization across layers. Overall, these results highlight the importance of selecting an appropriate model depth, rather than naively increasing the number of layers, in the design of the LENS architecture.

\subsection{Multi-step Analysis} \label{additional_exp_multi}

\begin{wraptable}{r}{0.45\textwidth}
\vspace{-12pt}
\caption{Step-wise performance}
\vspace{-2pt}
\centering
\scriptsize
\resizebox{\linewidth}{!}{
\begin{tabular}{lcc}
\toprule 
\textbf{SD-Turbo}  & \textbf{GenEval2 Mean $\uparrow$} \\
\midrule
One-step  & 33.95  \\
+ HyperNoise & 34.96  \\
\textbf{+ LENS (Ours)}  & \textbf{35.69}  \\
\midrule
Two-step  & 30.92  \\
+ HyperNoise & 32.95  \\
\textbf{+ LENS (Ours)} & \textbf{33.29}  \\
\midrule
Four-step  &  29.82  \\
+ HyperNoise  & \textbf{31.72}  \\
\textbf{+ LENS (Ours)}  & 31.00  \\
\midrule
Eight-step  &  29.16  \\
+ HyperNoise  & 30.58  \\
\textbf{+ LENS (Ours)}  & \textbf{30.61}  \\
\midrule
Sixteen-step  &  29.10  \\
+ HyperNoise  & \textbf{30.57}  \\
\textbf{+ LENS (Ours)}  & 30.10  \\
\bottomrule
\end{tabular}
}

\label{appendix_step}
\vspace{-50pt}
\end{wraptable}

While the main paper focuses on one-step generation, we further analyze performance across varying numbers of inference steps. Table~\ref{appendix_step} shows the performance trends as the number of generation steps increases. LENS consistently outperforms both random noise initialization and HyperNoise in terms of GenEval2 mean across all settings, except at 4 and 16 steps, where it remains comparable to HyperNoise while still surpassing random noise.

Notably, LENS achieves the largest performance gain in the one-step setting, attaining the highest score of 35.69. As the number of inference steps increases, performance slightly degrades. We attribute this trend to the design of the underlying model, SD-Turbo, which is specifically optimized through distillation for one-step generation. Consequently, while LENS remains competitive in multi-step settings, it achieves the best performance and efficiency when used with fewer inference steps.

\vspace{30pt}

\subsection{Results on Other Benchmarks} \label{additional_exp_other}
\paragraph{T2I-CompBench.}
Table~\ref{tab:t2i_compbench_results} presents quantitative evaluation results on T2I-CompBench~\cite{huang2023t2i}. LENS achieves competitive performance, closely matching HyperNoise with only a marginal gap, while significantly outperforming the Random Noise baseline. Notably, despite its substantially faster generation speed, LENS maintains comparable performance to HyperNoise, resulting in a more favorable performance--efficiency trade-off.

These results demonstrate that LENS provides an effective and resource-efficient solution for ensuring semantic alignment in text-to-image generation.
\begin{table}[!ht]
\vspace{-6pt}
\caption{Results on T2I-CompBench}
\vspace{3pt}
\centering
\resizebox{\textwidth}{!}{
\scriptsize
\begin{tabular}{l ccccc |c}
\toprule
\textbf{Model} & \textbf{Color} $\uparrow$ & \textbf{Shape} $\uparrow$ & \textbf{Texture} $\uparrow$ & \textbf{Spatial} $\uparrow$ & \textbf{Numeracy} $\uparrow$ & \textbf{Mean} $\uparrow$ \\
\midrule
SD-Turbo & 0.56 & 0.43 & 0.53 & 0.14 & 0.50 & 0.43 \\
+ HyperNoise & \textbf{0.58} & \textbf{0.46} & \textbf{0.57} & \textbf{0.19} &\textbf{0.53} & \textbf{0.47} \\
\textbf{+ LENS (Ours)} & 0.57 & \textbf{0.46} & 0.55 & 0.18 & 0.52 & 0.45 \\

\bottomrule
\end{tabular}
}

\label{tab:t2i_compbench_results}
\end{table}

%Table~\ref{tab:t2i_compbench_results} presents the quantitative evaluation results measured using T2I-CompBench~\cite{huang2023t2i}. The results indicate that LENS achieves competitive performance, nearly matching HyperNoise with only a marginal gap, while significantly outperforming the baseline Random Noise method. Given that LENS offers a substantially faster generation speed than HyperNoise, its ability to maintain comparable performance levels demonstrates a superior and highly practical performance-to-speed trade-off. These findings confirm that LENS is an effective and resource-efficient solution for ensuring semantic alignment in text-to-image generation.

\paragraph{DPG-Bench.}

\begin{wraptable}{r}{0.45\textwidth}
\vspace{-1pt}
\caption{Results on DPG-Bench}
\vspace{-2pt}
\centering
\scriptsize
\resizebox{\linewidth}{!}{
\begin{tabular}{lcc}
\toprule 
\textbf{Method}  & \textbf{DPG-Bench Score $\uparrow$} \\
\midrule
SD-Turbo  & 70.94  \\
+ HyperNoise & \textbf{72.40} \\
\textbf{+ LENS (Ours)}  & 72.33  \\
\bottomrule
\end{tabular}
}
\label{tab:dpg}
\end{wraptable}

Table~\ref{tab:dpg} presents the DPG-Bench~\cite{hu2024ella} scores for the SD-Turbo model. LENS achieves highly competitive performance, closely matching HyperNoise with only a marginal gap. Notably, despite its substantially faster generation speed, LENS maintains comparable semantic alignment to the strongest baseline, highlighting its superior practical efficiency.

\subsection{Aesthetic Quality Evaluation} \label{asthetic_eval}

To further evaluate perceptual detail and aesthetic quality, we conduct an additional evaluation using Q-Align~\cite{wu2024qalign}, an LMM-based perceptual metric aligned with human visual judgments. Q-Align Quality measures overall perceptual quality, while Q-Align Aesthetic evaluates aesthetic attributes such as color, lighting, and composition. We randomly sampled 100 prompts from PartiPrompts and generated images using SD-Turbo with unmodified noise, HyperNoise, and LENS.

\begin{table}[!ht]
\vspace{-6pt}
\caption{Q-Align evaluation}
\vspace{3pt}
\centering
\renewcommand{\arraystretch}{1.2}

\label{tab:qalign}
\begin{tabular}{lcc}
\toprule
Method & Q-Align Quality $\uparrow$ & Q-Align Aesthetic $\uparrow$ \\
\midrule
SD-Turbo     & 4.3197 & \textbf{3.0155} \\
+ HyperNoise & 4.4339 & 3.0054 \\
+ LENS       & \textbf{4.4672} & 3.0030 \\
\bottomrule
\end{tabular}
\end{table}

As shown in Table~\ref{tab:qalign}, LENS achieves the highest Q-Align Quality score, indicating improved perceptual quality. Its Q-Align Aesthetic score remains very close to those of unmodified noise and HyperNoise, differing by only 0.0125 and 0.0024, respectively. These results indicate that LENS improves perceptual quality while maintaining comparable aesthetic quality.

\subsection{Reward Component Ablation} \label{reward_ablation}
To analyze the contribution of each reward component, we conduct a leave-one-out ablation study on the four reward terms used by LENS. For each ablation, we remove one component while keeping all other training settings unchanged and train a new modulation network. All variants are evaluated on GenEval2 using the same prompt--seed pairs.

\begin{table}[!ht]
\centering

\renewcommand{\arraystretch}{1.2}
\vspace{-6pt}
\caption{Leave-one-out ablation of the reward components}
\vspace{3pt}
\label{tab:reward_ablation}
\resizebox{\linewidth}{!}{
\begin{tabular}{lccccccc}
\toprule
Method & Object $\uparrow$ & Attribute $\uparrow$ & Count $\uparrow$ &
Position $\uparrow$ & Verb $\uparrow$ & Mean $\uparrow$ & Mean drop \\
\midrule
LENS (w/ all rewards)      & 74.34 & 42.56 & 23.98 & 28.39 & 7.81 & 35.42 & -- \\
LENS (w/o CLIP)            & 73.89 & 42.22 & 23.76 & 28.39 & 3.95 & 34.44 & 0.98 \\
LENS (w/o HPS)             & 72.92 & 39.96 & 23.99 & 28.36 & 5.05 & 34.06 & 1.36 \\
LENS (w/o ImageReward)     & 73.55 & 41.43 & 24.03 & 28.56 & 7.26 & 34.97 & 0.45 \\
LENS (w/o PickScore)       & 74.22 & 41.96 & 24.26 & 27.29 & 8.52 & 35.25 & 0.17 \\
\bottomrule
\end{tabular}
}
\end{table}

As shown in Table~\ref{tab:reward_ablation}, using all four reward components achieves the best overall GenEval2 performance. Removing HPS results in the largest drop in the mean score, followed by CLIP, ImageReward, and PickScore. Although the contribution of each component varies across individual categories, the full reward yields the highest overall performance, suggesting that the four reward terms play complementary roles.

\subsection{Sensitivity to the PCA Basis Dataset} \label{pca_data_ablation}
To evaluate the sensitivity of LENS to the dataset used for constructing the PCA basis, we additionally compute PCA bases from Open Images v7~\cite{kuznetsova2020open} and Places365~\cite{zhou2017places}, each using 10,000 uniformly sampled images, in addition to the ImageNetV2 basis used in the main experiments. For each basis, we train a separate LENS modulation network using the same SD-Turbo training configuration and evaluate all variants on GenEval2 with identical prompt--seed pairs.

\begin{table}[!ht]
\vspace{-6pt}
\caption{Sensitivity of LENS to the dataset used for constructing the PCA basis}
\vspace{3pt}
\centering
\renewcommand{\arraystretch}{1.2}
\label{tab:pca_dataset_sensitivity}
\begin{tabular}{lcccccc}
\toprule
Method & Object $\uparrow$ & Attribute $\uparrow$ & Count $\uparrow$ &
Position $\uparrow$ & Verb $\uparrow$ & Mean $\uparrow$ \\
\midrule
SD-Turbo & 72.14 & 42.72 & 20.42 & 29.50 & 5.82 & 34.12 \\
+ LENS (ImageNetV2 basis) & 74.34 & 42.56 & 23.98 & 28.39 & 7.81 & 35.42 \\
+ LENS (Open Images basis) & 74.90 & 44.12 & 24.50 & 29.97 & 7.36 & 36.17 \\
+ LENS (Places365 basis) & 75.12 & 41.34 & 24.51 & 27.59 & 6.59 & 35.03 \\
\bottomrule
\end{tabular}
\end{table}

As shown in Table~\ref{tab:pca_dataset_sensitivity}, LENS improves the overall GenEval2 mean over the SD-Turbo baseline for all three PCA bases. Although the scores vary slightly across individual categories, the overall improvement is preserved regardless of the dataset used to construct the PCA basis. These results indicate that LENS is not strongly sensitive to the choice of PCA basis dataset.

\subsection{Standard Deviations of GenEval2 Scores} \label{appendix_std}
Table~\ref{tab:geneval_results} reports GenEval2 results obtained by averaging over three runs  with different base seeds to mitigate dependence on specific random seeds. Table~\ref{tab:geneval_std} provides the corresponding standard deviations.

\begin{table}[!ht]
\vspace{-6pt}
\caption{Standard deviations of GenEval2 results in Table~\ref{tab:geneval_results}}
\vspace{3pt}
\centering
\resizebox{\textwidth}{!}{
\tiny
\renewcommand{\arraystretch}{1.0}
\begin{tabular}{l  ccccc |c}
\toprule
\textbf{Model}  & \textbf{Object} & \textbf{Attribute} & \textbf{Count} & \textbf{Position} & \textbf{Verb} & \textbf{Mean} \\
\midrule
SD-Turbo & 0.29 & 0.30 & 0.21 & 0.32 & 0.21 & 0.24 \\
+ ReNo & 0.38 & 0.42 & 0.33 & 0.37 & 0.27 & 0.35 \\
+ HyperNoise  & 0.40 & 0.48  & 0.39 & 0.36 & 0.25 & 0.33 \\
+ LENS (Ours) & 0.37 & 0.40 & 0.36 & 0.40 & 0.23 & 0.35 \\

\midrule
Hyper-SDXL & 0.40 & 0.43 & 0.32 & 0.29 & 0.20 & 0.32 \\
+ ReNo & 0.45 & 0.43 & 0.35 & 0.37 & 0.16 & 0.35 \\
+ HyperNoise  & 0.40 & 0.44  & 0.38 & 0.39 & 0.25 & 0.32 \\
+ LENS (Ours) & 0.42 & 0.41 & 0.36 & 0.33 & 0.21 & 0.35 \\

\midrule
SANA-Sprint & 0.47 & 0.44 & 0.41 & 0.46 & 0.56 & 0.42 \\
+ ReNo & 0.44 & 0.43 & 0.42 & 0.52 & 0.53 & 0.47 \\
+ HyperNoise  & 0.43 & 0.44  & 0.38 & 0.50 & 0.74 & 0.45 \\
+ LENS (Ours) & 0.41 & 0.43 & 0.43 & 0.49 & 0.61 & 0.48 \\

\bottomrule
\end{tabular}
}

\label{tab:geneval_std}
\end{table}

\section{Broader Impact} \label{broader}
Generative models have been widely studied and increasingly adopted across various domains, leading to substantial societal impact. This work contributes to improving the quality and computational efficiency of such models, enabling higher-quality generation with reduced computational cost. In this regard, our approach may facilitate broader accessibility to generative modeling, particularly in resource-constrained settings.

At the same time, we do not expect our work to introduce fundamentally new societal risks beyond those already associated with existing generative models. However, inherent risks such as the potential misuse for generating misleading or synthetic content, including misinformation, remain relevant. These concerns are not unique to our method but are shared across the broader class of generative modeling techniques. Responsible deployment and appropriate usage guidelines are therefore important to mitigate such risks.

\end{document}